\documentclass{article}

 \usepackage[final,nonatbib]{neurips_2020}

\usepackage[utf8]{inputenc} %
\usepackage[T1]{fontenc}    %
\usepackage{hyperref}       %
\usepackage{url}            %
\usepackage{booktabs}       %
\usepackage{amsfonts}       %
\usepackage{nicefrac}       %
\usepackage{microtype,array}      %
\usepackage{float}
\usepackage{graphicx}
\usepackage{subfig}
\usepackage{multirow}
\usepackage{xcolor}
\usepackage{appendix}
\usepackage{xfrac}
\usepackage{booktabs}
\usepackage{amsmath,amssymb,amsthm}
\usepackage{dsfont,xspace,enumitem}
\usepackage{caption}
\usepackage{hyperref}
\usepackage{cleveref}
\usepackage{wrapfig}

\newtheorem{theorem}{Theorem}
\newtheorem{lemma}{Lemma}
\newtheorem{retheorem}{Theorem}
\newtheorem{fact}{Fact}

\theoremstyle{definition}
\newtheorem{definition}{Definition}

\newcommand{\lms}{\texttt{LMS}\xspace}

\newcommand{\mnistcifar}{\texttt{MNIST-CIFAR}\xspace}
\newcommand{\mnistcifara}{\texttt{MNIST-CIFAR:A}\xspace}
\newcommand{\mnistcifarb}{\texttt{MNIST-CIFAR:B}\xspace}
\newcommand{\mnistcifarc}{\texttt{MNIST-CIFAR:C}\xspace}
\newcommand{\mnist}{\texttt{MNIST}\xspace}
\newcommand{\cifar}{\texttt{CIFAR10}\xspace}
\newcommand{\lsn}{\texttt{LSN}\xspace}
\newcommand{\lmsk}{\texttt{LMS-k}\xspace}
\newcommand{\lmsfive}{\texttt{LMS-5}\xspace}
\newcommand{\lmsseven}{\texttt{LMS-7}\xspace}
\newcommand{\nlmsk}{\texttt{\^LMS-k}\xspace}
\newcommand{\nlmsseven}{\texttt{\^LMS-7}\xspace}
\newcommand{\nlmsfive}{\texttt{\^LMS-5}\xspace}
\newcommand{\msfiveseven}{\texttt{MS-(5,7)}\xspace}
\newcommand{\advms}{\texttt{AdvMS-(5,7)}\xspace}
\newcommand{\nmsfiveseven}{\texttt{\^MS-(5,7)}\xspace}
\newcommand{\msfive}{\texttt{MS-5}\xspace}
\newcommand{\msseven}{\texttt{MS-7}\xspace}
\newcommand{\SF}{$\mathtt{S}$\xspace}
\newcommand{\SFc}{$\mathtt{S}^c$\xspace}

\newcommand{\order}[1]{O\left(#1\right)}
\newcommand{\minus}{\scalebox{0.75}[1.0]{$-$}}

\newcommand{\abs}[1]{\left|#1\right|}
\newcommand{\relu}{\text{ReLU}}

\newcommand{\pr}{\mathds{P}}
\renewcommand{\S}{\mathcal{S}}
\newcommand{\PR}[1]{\mathds{P}(#1)}
\newcommand{\loss}{\mathcal{L}}
\newcommand{\floor}[1]{\lfloor #1 \rfloor}
\newcommand{\bigo}{{O}}

\newcommand{\D}{\mathcal{D}}
\newcommand{\DRS}{\overline{\mathcal{D}}^S}

\newcommand{\R}{\mathbb{R}}
\newcommand{\Rd}{\mathbb{R}^d}
\newcommand{\E}[1]{\mathbb{E}\left[#1\right]}
\newcommand{\ED}[1]{\mathbb{E}_{\D}\left[#1\right]}
\newcommand{\EDRS}[1]{\mathbb{E}_{\DRS}\left[#1\right]}

\newcommand{\ind}[1]{\mathds{1}{\left\{#1\right\}}}
\newcommand{\norm}[1]{\left\| #1 \right\|}
\newcommand{\grad}{\nabla}
\newcommand{\G}{\mathcal{G}}

\title{ The Pitfalls of Simplicity Bias in Neural Networks}

\author{
	Harshay Shah\\
	Microsoft Research\\
	\texttt{harshay.rshah@gmail.com}\\
	\And
	Kaustav Tamuly\\
	Microsoft Research\\
	\texttt{ktamuly2@gmail.com}\\
	\And
	Aditi Raghunathan\\
	Stanford University\\
	\texttt{aditir@stanford.edu}\\
	\And
	Prateek Jain\\
	Microsoft Research\\
	\texttt{prajain@microsoft.com} \\
	\And
	Praneeth Netrapalli\\
	Microsoft Research\\
	\texttt{praneeth@microsoft.com}\\
}

\begin{document}

\maketitle

\begin{abstract}
Several works have proposed Simplicity Bias (SB)---the tendency of standard training procedures such as Stochastic Gradient Descent (SGD) to find simple models---to justify why neural networks generalize well~\cite{arpit2017closer,nakkiran2019sgd,valle-perez2018deep}.
However, the precise notion of simplicity remains vague. Furthermore, previous settings~\cite{soudry2018implicit,gunasekar2018implicit} that use SB to justify why neural networks generalize well do not simultaneously capture the non-robustness of neural networks---a widely observed phenomenon in practice~\cite{szegedy2013intriguing,jo2017measuring}.
We attempt to reconcile SB and the superior standard generalization of neural networks with the non-robustness observed in practice by designing datasets that (a) incorporate a precise notion of simplicity, (b) comprise multiple predictive features with varying levels of simplicity, and (c) capture the non-robustness of neural networks trained on real data.
Through theoretical analysis and targeted experiments on these {datasets}, we make four observations:\\ 
(i) SB of SGD and variants can be extreme: neural networks can exclusively rely on the simplest feature and remain invariant to all predictive complex features. \\
(ii) The extreme aspect of SB {\em could} explain why seemingly benign distribution shifts and small adversarial perturbations significantly degrade model performance.\\
(iii) Contrary to conventional wisdom, SB can also hurt generalization on the same data distribution, as SB persists even when the simplest feature has less predictive power than the more complex features.\\
(iv) Common approaches to improve generalization and robustness---ensembles and adversarial training---can fail in  mitigating SB and its pitfalls.\\
 Given the role of SB in training neural networks, we hope that the proposed datasets and methods serve as an effective testbed to evaluate novel algorithmic approaches aimed at avoiding the pitfalls of SB.

\end{abstract}

\section{Introduction}
\label{sec:intro}
\vspace{-0.2cm}
Understanding the superior generalization ability of neural networks, despite their high capacity to fit randomly labeled data~\cite{zhang2017understanding}, has been a subject of intense study.
One line of recent work~\cite{soudry2018implicit,gunasekar2018implicit} proves that linear neural networks trained with Stochastic Gradient Descent (SGD) on linearly separable data converge to the maximum-margin linear classifier, thereby explaining the superior generalization performance. 
However, maximum-margin classifiers are inherently robust to perturbations of data at prediction time, and this implication is at odds with concrete evidence that neural networks, in practice, are brittle to adversarial examples~\cite{szegedy2013intriguing} and distribution shifts~\cite{oakden2019hidden, ribeiro2016should, mccoy2019right, simpson2019gradmask}. 
Hence, the linear setting, while convenient to analyze, is insufficient to capture the non-robustness of neural networks trained on real datasets.
Going beyond the linear setting, several works~\cite{arpit2017closer,nakkiran2019sgd,valle-perez2018deep} argue that neural networks generalize well because standard training procedures have a bias towards learning simple models. 
However, the exact notion of ``simple" models remains vague and only intuitive. Moreover, the settings studied are insufficient to capture the brittleness of neural networks

Our goal is to formally understand and probe the \emph{simplicity bias} (SB) of neural networks in a setting that is rich enough to capture known failure modes of neural networks and, at the same time, amenable to theoretical analysis and targeted experiments. Our starting point is the observation that on real-world datasets, there are several distinct ways to discriminate between labels (e.g., by inferring shape, color etc. in image classification) that are (a) predictive of the label to varying extents, and (b) define decision boundaries of varying complexity. For example, in the image classification task of white swans vs. bears, a linear-like ``simple" classifier that only looks at color could predict correctly on most instances except white polar bears, while a nonlinear ``complex" classifier that infers shape could have almost perfect predictive power.
To systematically understand SB, we design modular synthetic and image-based datasets wherein different coordinates (or blocks) define decision boundaries of varying complexity. We refer to each coordinate / block as a \emph{feature} and define a precise notion of feature \emph{simplicity} based on the \emph{simplicity of the corresponding decision boundary}.

\vspace{-7px}
\paragraph{Proposed dataset.} 
~\Cref{fig:intro} illustrates a stylized version of the proposed synthetic dataset with two features, $\phi_1$ and $\phi_2$, that can perfectly predict the label with 100\% accuracy, but differ in simplicity.
\begin{wrapfigure}{r}{0.375\linewidth}
	\centering
	\vspace{-10px}
	\includegraphics[width=\linewidth]{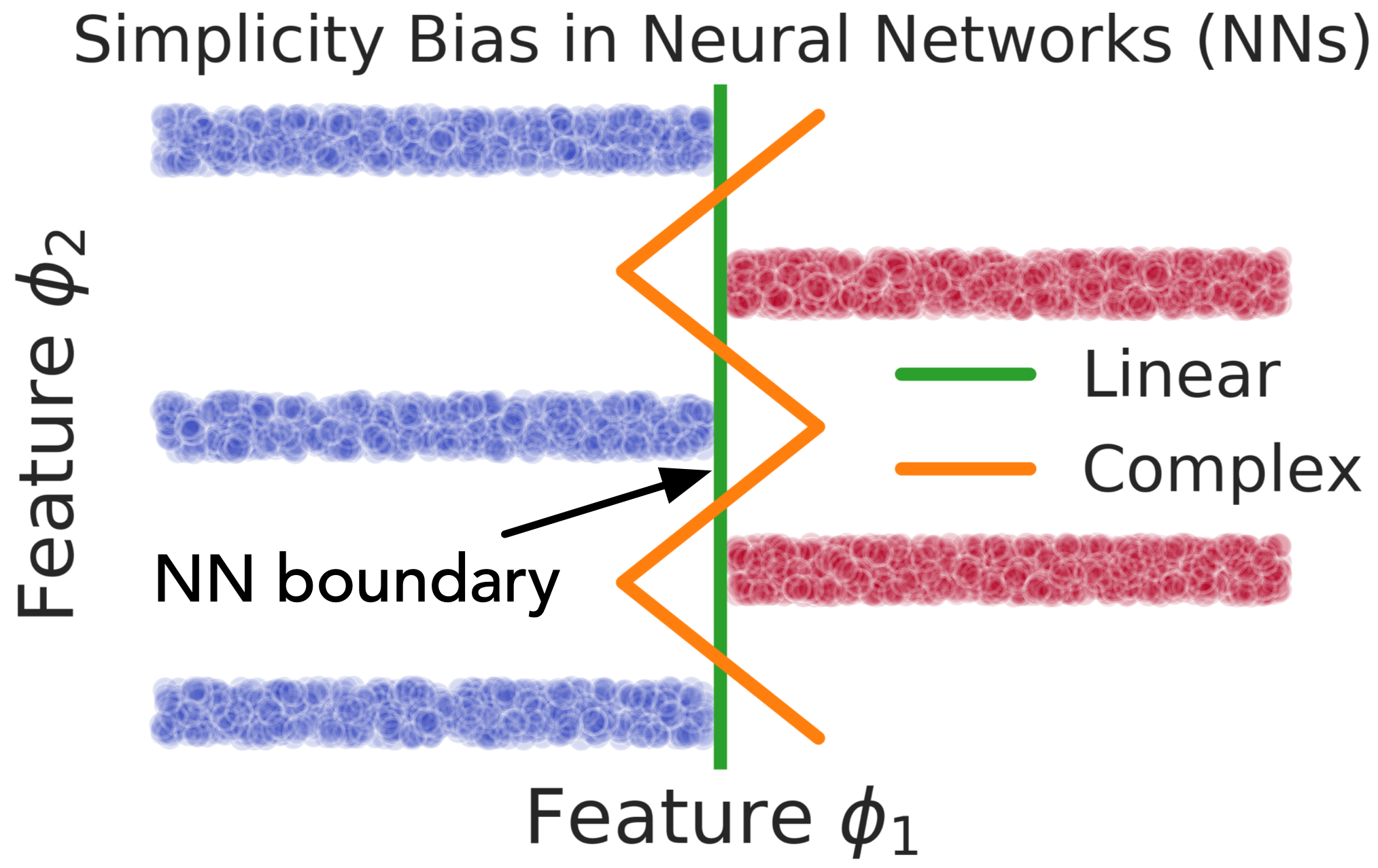}
	\caption{Simple vs. complex features}
	\label{fig:intro}
        \vspace{-15px}
\end{wrapfigure}
The simplicity of a feature is precisely determined by the \emph{minimum} number of linear pieces in the decision boundary that achieves optimal  classification accuracy using that feature. 
For example, in Figure~\ref{fig:intro}, the simple feature $\phi_1$ requires a linear decision boundary to perfectly predict the label, whereas complex feature $\phi_2$ requires four linear pieces. 
Along similar lines, we also introduce a collection of image-based datasets in which each image concatenates \texttt{MNIST} images (simple feature) and 
	\texttt{CIFAR-10} images (complex feature). 
The proposed datasets, which incorporate features of varying predictive power and simplicity, allow us to systematically investigate and measure SB in SGD-trained neural networks.

\vspace{-10px}
\paragraph{Observations from new dataset.} 
The ideal decision boundary that achieves high accuracy \emph{and} robustness relies on all features to obtain a large margin (minimum distance from any point to decision boundary).
For example, the orange decision boundary in~\Cref{fig:intro} that learns $\phi_1$ and $\phi_2$ attains 100\% accuracy and exhibits more robustness than the linear boundary because of larger margin.  	
Given the expressive power of large neural networks, 
one might expect that a network trained on the dataset in~\Cref{fig:intro} would result in the larger-margin orange piecewise linear boundary. However, in practice, we find quite the opposite---trained neural networks have a linear boundary. Surprisingly, neural networks exclusively use feature $\phi_1$ and remain \emph{completely invariant} to $\phi_2$. More generally, we observe that SB is extreme: neural networks simply ignore several complex predictive features in the presence of few simple predictive features. 
We first theoretically show that one-hidden-layer neural networks trained on the piecewise linear dataset exhibit SB. 
Then, through controlled experiments, we validate the extreme nature of SB across model architectures and optimizers.

\vspace{-7px}
\paragraph{Implications of extreme SB.} Theoretical analysis and controlled experiments reveal three major pitfalls of SB in the context of proposed synthetic and image-based datasets, which we {\em conjecture} to hold more widely across datasets and domains:

(i) \underline{Lack of robustness}:  Neural networks exclusively latch on to the simplest feature (e.g., background) at the expense of very small margin and completely ignore complex predictive features (e.g., semantics of the object), \emph{even when all features have equal predictive power}. This results in susceptibility to small adversarial perturbations (due to small margin) and spurious correlations (with simple features). Furthermore, in~\Cref{sec:extreme}, we provide a concrete connection between SB and data-agnostic and model-agnostic universal adversarial perturbations ~\cite{moosavi2017universal} observed in practice.

(ii) \underline{Lack of reliable confidence estimates}: Ideally, a network should have high confidence only if all predictive features agree in their prediction. However due to extreme SB, the network has high confidence even if several complex predictive features contradict the simple feature, mirroring the widely reported inaccurate and substantially higher confidence estimates reported in practice~\cite{nguyen2015deep, guo2017calibration}. 

(iii) \underline{Suboptimal generalization}: Surprisingly, neural networks exclusively rely on the simplest feature \emph{even if it less predictive of the label than all complex features} in the synthetic datasets. Consequently, contrary to conventional wisdom, extreme SB can hurt robustness as well as generalization.

In  contrast, prior works~\cite{brutzkus2017sgd,soudry2018implicit,gunasekar2018implicit} only extol SB by considering settings where all predictive features are simple and hence do not reveal the pitfalls observed in real-world settings.
{While our results on the pitfalls of SB are established in the context of the proposed datasets, the two design principles underlying these datasets---combining multiple features of varying simplicity \& predictive power and capturing multiple failure modes of neural networks in practice---suggest that our conclusions {\em could} be justifiable more broadly.}

\vspace{-12px}
\paragraph{Summary.} This work makes two key contributions. 
{First, we design datasets that offer a precise stratification of features based on simplicity and predictive power.
Second, using the proposed datasets, we provide theoretical and empirical evidence that neural networks exhibit extreme SB, which we postulate as a unifying contributing factor underlying key failure modes of deep learning: poor out-of-distribution performance, adversarial vulnerability and suboptimal generalization. 
To the best of our knowledge, prior works only focus on the positive aspect of SB: the lack of overfitting in practice.
Additionally, we find that standard approaches to improve generalization and robustness---ensembles and adversarial training---do not mitigate simplicity bias and its shortcomings on the proposed datasets.
Given the important implications of SB, we hope that the datasets we introduce serve (a) as a useful testbed for devising better training procedures and (b) as a starting point to design more realistic datasets that are amenable to theoretical analysis and controlled experiments.}

\textbf{Organization.} We discuss related work in~\Cref{sec:related}. ~\Cref{sec:prelims} describes the proposed datasets and metrics. In~\Cref{sec:extreme}, we concretely establish the extreme nature of Simplicity Bias (SB) and its shortcomings through theory and empirics.
	~\Cref{sec:gen} shows that extreme SB can in fact hurt generalization as well. 
	We conclude and discuss the way forward in~\Cref{sec:discussion}.
\vspace{-4px}

\vspace{-5px}
\section{Related Work}
\label{sec:related}
\vspace{-8px}

{\bf Out-of-Distribution (OOD) performance}: Several works demonstrate that NNs tend to learn spurious features \& low-level statistical patterns rather than semantic features \& high-level abstractions, resulting in poor OOD performance~\cite{jo2017measuring,geirhos2018imagenet,mccoy-etal-2019-right,oakden2019hidden}. 
This phenomenon has been exploited to design backdoor attacks against NNs~\cite{biggio2012poisoning,chen2017targeted} as well.
Recent works~\cite{wang2018learning,wang2019learning} that encourage models to learn higher-level features improve OOD performance, but require domain-specific knowledge to penalize reliance on spurious features such as image texture~\cite{geirhos2018imagenet} and annotation artifacts~\cite{gururangan2018annotation} in vision \& language tasks.
Learning robust representations without domain knowledge, however, necessitates formalizing the notion of features and feature reliance;
	our work takes a step in this direction.

\vspace{-2px}
{\bf Adversarial robustness}: 
	Neural networks exhibit vulnerability to small adversarial perturbations~\cite{szegedy2013intriguing}. Standard approaches to mitigate this issue---adversarial training~\cite{goodfellow2014explaining,madry2017towards} and ensembles~\cite{strauss2017ensemble,pang2019improving,kariyappa2019improving}---have had limited success on large-scale datasets.
Consequently, several works have investigated reasons underlying the existence of adversarial examples:~\cite{goodfellow2014explaining} suggests local linearity of trained NNs,~\cite{schmidt2018adversarially} indicates insufficient data,~\cite{shafahi2018adversarial} suggests inevitability in high dimensions,~\cite{bubeck2018adversarial} suggests computational barriers,~\cite{degwekar2019computational} proposes limitations of neural network architectures, and \cite{ilyas2019adversarial} proposes the presence of non-robust features. Additionally, Jacobsen et al.~\cite{jacobsen2018excessive} show that NNs exhibit invariance to large label-relevant perturbations.
Prior works have also demonstrated the existence of \emph{universal adversarial perturbations} (UAPs) that are agnostic to model and data~\cite{papernot2016transferability,moosavi2017universal,tramer2017space}.

\vspace{-2px}
{\bf Implicit bias of stochastic gradient descent }: Brutzkus et al.~\cite{brutzkus2017sgd} show that neural networks trained with SGD provably generalize on linearly separable data. 
Recent works~\cite{soudry2018implicit,ji2019implicit} also analyze the limiting direction of gradient descent on logistic regression with linearly separable and non-separable data respectively. 
Empirical findings~\cite{nakkiran2019sgd,mangalam2019deep} provide further evidence to suggest that SGD-trained NNs generalize well because SGD learns models of increasing complexity.
Additional recent works investigate the implicit bias of SGD on non-linearly separable data for linear classifiers~\cite{ji2019implicit} and infinite-width two-layer NNs~\cite{chizat2020implicit}, showing convergence to maximum margin classifiers in appropriate spaces. 
In~\Cref{sec:extreme}, we show that SGD's implicit bias towards simplicity can result in  small-margin and feature-impoverished classifiers instead of large-margin and feature-dense classifiers.

\vspace{-2px}
\textbf{Feature reliance}: %
Two recent works study the relation between inductive biases of training procedures and the set of features that models learn. 
Hermann et al.~\cite{hermann2020shapes} use color, shape and texture features in stylized settings to show that standard training procedures can (a) increase reliance on task-relevant features that are partially decodable using untrained networks and (b) suppress reliance on non-discriminative or correlated features. 
Ortiz et al.~\cite{ortiz2020hold} show that neural networks learned using standard training (a) develop invariance to non-discriminative features and (b) adversarial training induces a sharp transition in the models' decision boundaries.
 In contrast, we develop a \emph{precise} notion of feature simplicity and subsequently show that SGD-trained models can exhibit invariance to multiple \emph{discriminative-but-complex} features. 
 We also identify three pitfalls of this phenomenon---poor OOD performance, adversarial vulnerability, suboptimal generalization---and show that adversarial training and standard ensembles do not mitigate the pitfalls of simplicity bias.
	
Multiple works mentioned above (a) differentially characterize \emph{learned} features and \emph{desired} features---statistical regularities vs. high-level concepts~\cite{jo2017measuring}, syntactic cues vs. semantic meaning~\cite{mccoy-etal-2019-right},  robust vs. non-robust features~\cite{ilyas2019adversarial}---and (b) posit that the mismatch between these 
	features results in non-robustness.
In contrast, our work probes \emph{why} neural networks prefer one set of features over another and unifies the aforementioned feature characterizations through the lens of feature \emph{simplicity}.

\section{Preliminaries: Setup and Metrics}
\label{sec:prelims}
\vspace{-10px}
{\bf Setting and metrics}: We focus on binary classification. Given samples $\widehat{\D}=\{(x_i,y_i)\}^n_{i=1}$ from distribution $\D$ over $\Rd \times \{\minus 1,1\}$,
the goal is to learn a scoring function $s(x):\Rd\rightarrow \R$ (such as logits), and an associated classifier $f:\Rd \rightarrow \{\minus 1,1\}$ defined as $f(x)=2h(x)\minus 1$ where $h(x)=\mathds{1}\{\text{softmax}(s(x)) < 0.5\}$. We use two well-studied metrics for generalization and robustness:

\vspace{-10px}
\paragraph{Standard accuracy.} The standard accuracy of a classifier $f$ is: $\ED{\ind{f(x) = y}}$.

\vspace{-10px}
\paragraph{$\delta$-Robust accuracy.} Given norm $\norm{\cdot}$ and perturbation budget $\delta$, the $\delta$-robust accuracy of a classifier $f$ is: $\ED{\min_{\norm{\hat{x} - x}\leq \delta} \ind{f(\hat{x}) = y}}$.

Next, we introduce two metrics that quantitatively capture the extent to which a model relies on different input coordinates (or features). Let $S$ denote some subset of coordinates $[d]$ and $\DRS$ denote the $S$-randomized distribution, which is obtained as follows: given $\D^S$, the marginal distribution of $S$, $\DRS$ independently samples $((x^S, x^{S^\mathsf{c}}),y) \sim \D$ and $\overline{x}^S \sim \D^S$ and then outputs $((\overline{x}^S, x^{S^\mathsf{c}}),y)$. In $\DRS$, the coordinates in $S$ are rendered independent of the label $y$. The two metrics are as follows. 
\begin{definition}[Randomized accuracy]\label{def:ran_acc}
	Given data distribution $\D$, and subset of coordinates $S \subseteq [d]$, the $S$-randomized accuracy of a classifier $f$ is given by: $\EDRS{\ind{f(x) = y}}$.
\end{definition}
\begin{definition}[Randomized AUC]
	Given data distribution $\D$ and subset of coordinates $S \subseteq [d]$, the $S$-randomized AUC of classifier $f$ equals the area under the precision-recall curve of distribution $\DRS$.
\end{definition}
\vspace{-4px}
Our experiments use \{\SF, \SFc\}-randomized metrics---accuracy, AUC, logits---to establish that $f$ depends \emph{exclusively on some features \SF} and \emph{remains invariant to the rest \SFc}.
\begin{table}[t]
	\centering
	\def\arraystretch{0.95}
	\resizebox{.77\linewidth}{!}{%
	\begin{tabular}{cccc}
		\toprule
		& Accuracy & AUC & Logits \\ \midrule
		\SF-Randomized & 0.50 & 0.50 & randomly shuffled \\
		\SFc-Randomized & standard accuracy & standard AUC & essentially identical \\ \bottomrule	
	\end{tabular}%
	}
	\parbox{11cm}{
		\caption{If the \{\SF, \SFc\}-randomized metrics of a model behave as above, then that model relies \emph{exclusively} on \SF and is \emph{invariant} to \SFc.}
		\label{tab:metrics}
	}
	\vspace{-7px}
\end{table}
First, if (a) \SF-randomized accuracy and AUC equal $0.5$ and (b) \SF-randomized logit distribution is a random shuffling of the original distribution (i.e., logits in the original distribution are randomly shuffled across true positives and true negatives), then $f$ depends \emph{exclusively} on \SF. Conversely, if (a) \SFc-randomized accuracy and AUC are equal to standard accuracy and AUC and (b) \SFc-randomized logit distribution is essentially identical to the original distribution, then $f$ is \emph{invariant} to \SFc; ~\Cref{tab:metrics} summarizes these observations.

\vspace{-7px}
\subsection{Datasets}
\vspace{-2px}
\textbf{One-dimensional Building Blocks}:
\label{subsection:one-dim-blocks} 
Our synthetic datasets use three one-dimensional data blocks---linear, noisy linear and $k$-slabs---shown in top row of~\Cref{fig:datasets}. In the \textbf{linear} block, positive and negative examples are uniformly distributed in $[0.1, 1]$ and $[\minus 1, \minus 0.1]$ respectively. In the \textbf{noisy linear} block, given a noise parameter $p \in [0,1]$, $1-p$ fraction of points are distributed like the linear block described above and $p$ fraction of the examples are uniformly distributed in $[\minus 0.1, 0.1]$. In \textbf{$k$-slab} blocks, positive and negative examples are distributed in $k$ well-separated, alternating regions. 

\textbf{Simplicity of Building Blocks}: Linear classifiers can attain the optimal (Bayes) accuracy of $1$ and $1-\sfrac{p}{2}$ on the linear and $p$-noisy linear blocks respectively. For $k$-slabs, however, $(k\minus 1)$-piecewise linear classifiers are required to obtain the optimal accuracy of $1$.  Consequently, the building blocks have a natural notion of simplicity: \emph{minimum number of pieces required by a piecewise linear classifier to attain optimal accuracy}. With this notion, the linear and noisy linear blocks are simpler than $k$-slab blocks when $k > 2$, and $k$-slab blocks are simpler than $\ell$-slab blocks when $k < \ell$. 
\begin{figure}[t]
	\centering
	\includegraphics[width=\linewidth]{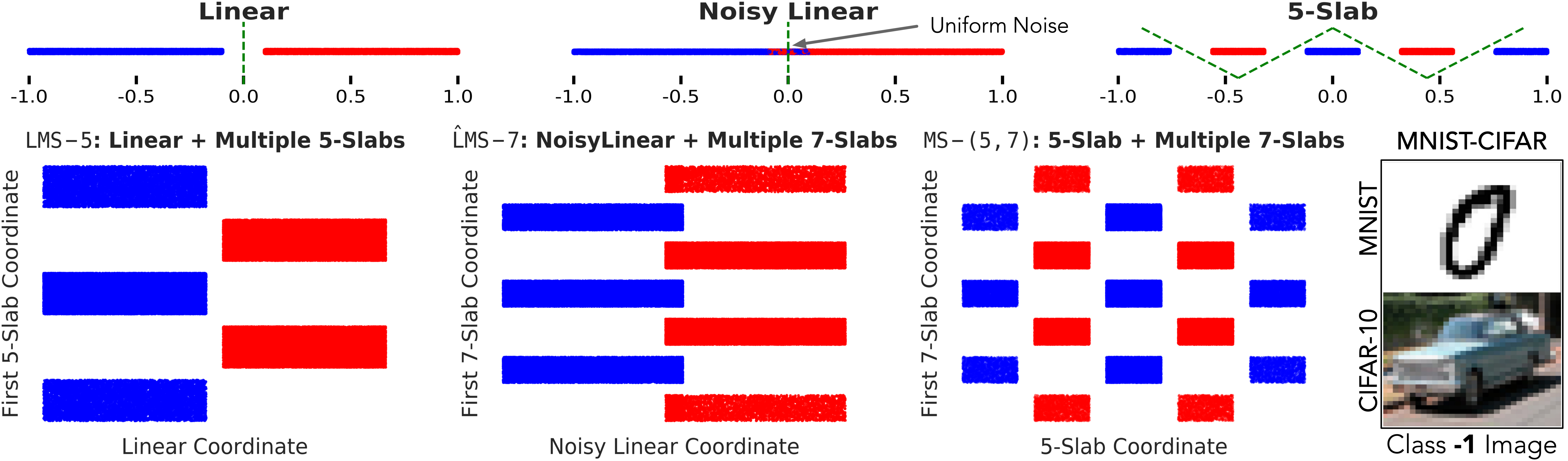}
	\caption{(Synthetic \& Image-based Datasets)
		One-dimensional building blocks (top row)---linear, noisy linear, $k$-slab---are used to construct multi-dimensional datasets (bottom row): \texttt{LMS-5} (linear \& multiple $5$-slabs), $\texttt{\^LMS-5}$ (noisy linear \& multiple $5$-slabs) and \texttt{MS-(5,7)} ($5$-slab \& multiple $7$-slabs). ~\mnistcifar data vertically concatenates \texttt{MNIST} and \texttt{CIFAR} images (see~\Cref{subsection:mnistcifar}).
	}
	\label{fig:datasets}	
\end{figure}

\textbf{Multi-dimensional Synthetic Datasets}:
\label{sec:synth}
We now outline four $d$-dimensional datasets wherein each coordinate corresponds to one of three building blocks described above. See Figure~\ref{fig:datasets} for illustration.
\vspace{-3px}
\begin{itemize}[leftmargin=*]
	\itemsep 2pt
	\topsep -5pt
	\parskip 0pt
	\item \underline{\texttt{LMS-k}}: Linear and multiple $k$-slabs; the first coordinate is a linear block and the remaining $d\minus 1$ coordinates are independent $k$-slab blocks; we use \texttt{LMS-5} \& \texttt{LMS-7} datasets in our analysis.
	\item \underline{\texttt{\^LMS-k}}: Noisy linear and multiple $k$-slab blocks; the first coordinate is a \emph{noisy} linear block and the remaining $d\minus 1$ coordinates are independent $k$-slab blocks. The noise parameter $p$ is $0.1$ by default.
	\item \underline{\texttt{MS-(5,7)}}: 5-slab and multiple 7-slab blocks; the first coordinate is a 5-slab block and the remaining $d\minus 1$ coordinates are independent $7$-slab blocks, as shown in~\Cref{fig:datasets}.
	\item \underline{\texttt{MS-5}}: Multiple 5-slab blocks; all coordinates are independent 5-slab blocks. 
\end{itemize}

\vspace{-2px}
We now describe the \lsn (linear, 3-slab \& noise) dataset, a stylized version of \lmsk that is amenable to theoretical analysis. 
We note that recent works~\cite{ortiz2020hold,gardner2020evaluating} {empirically} analyze variants of the \lsn dataset.
In \lsn, conditioned on the label $y$, the first and second coordinates of $x$ are \emph{singleton} linear and $3$-slab blocks: linear and $3$-slab blocks have support on $\{\minus 1,1\}$ and $\{\minus 1,0,1\}$ respectively. The remaining coordinates are standard gaussians and not predictive of the label. 

The synthetic datasets comprise features of varying simplicity; in~\lmsk, \nlmsk, and~\msfiveseven, the first coordinate is the simplest feature and in~\msfive, all features are equally simple. All datasets, even \nlmsk, can be \emph{perfectly} classified via piecewise linear classifiers.
Though the $k$-slab features are special cases of linear periodic functions on which gradient-based methods have been shown to fail for large $k$~\cite{shalev2017failures}, we note that we use small values of $k \in \{5,7\}$ which are quickly learned by SGD in practice.
Note that we (a) apply a random rotation matrix to the data and (b) use $50$-dimensional synthetic data (i.e., $d=50$) by default. Note that all code and datasets are available at the following repository: \href{https://github.com/harshays/simplicitybiaspitfalls}{https://github.com/harshays/simplicitybiaspitfalls}.

\textbf{MNIST-CIFAR Data}:
\label{subsection:mnistcifar}
The \mnistcifar dataset consists of two classes: images in class $\minus 1$ and class $1$ are vertical concatenations of \texttt{MNIST} digit zero \& \texttt{CIFAR-10} automobile and \texttt{MNIST} digit one \& \texttt{CIFAR-10} truck images respectively, as shown in~\Cref{fig:datasets}. The training and test datasets comprise 50,000 and 10,000 images of size $3\times 64\times 32$.
The \mnistcifar dataset mirrors the structure in the synthetic \lmsk dataset---both incorporate simple and complex features. 
The \texttt{MNIST} and \texttt{CIFAR} blocks correspond to the linear and $k$-slab blocks in \lmsk respectively.
Also note that \texttt{MNIST} images are zero-padded \& replicated across three channels to match \texttt{CIFAR} dimensions before concatenation.

\Cref{sec:setup} provides details about the datasets, models, and optimizers used in our experiments. In~\Cref{sec:ood}, we show that our results are robust to the exact choice of~\mnistcifar class pairs.

\vspace{-5px}
\section{Simplicity Bias (SB) is Extreme and Leads to Non-Robustness}
\label{sec:extreme}
\vspace{-10px}
We first establish the \emph{extreme} nature of SB in neural networks (NNs) {on the proposed synthetic datasets} using SGD and variants.
In particular, we show that for the datasets considered, \emph{if all features have full predictive power, NNs rely exclusively on the simplest feature \SF and remain invariant to all complex features \SFc}.
Then, we explain why extreme SB {on these datasets} results in neural networks that are vulnerable to distribution shifts and data-agnostic \& transferable adversarial perturbations.

\vspace{-8px}
\subsection{Neural networks provably exhibit Simplicity Bias (SB)}
We consider the \lsn dataset (described in~\Cref{sec:synth}) that has one {\em linear} coordinate and one {\em 3-slab} coordinate, both fully predictive of the label on their own; the remaining $d\minus 2$ noise coordinatesI  do not have any predictive power. Now, a "large-margin" one-hidden-layer NN with ReLU activation should give equal weight to the linear and 3-slab coordinates.
	However,  we prove that NNs trained with standard mini-batch gradient descent (GD) on the \lsn dataset (described in~\Cref{sec:synth}) provably learns a classifier that \emph{exclusively} relies on the ``simple" linear coordinate, thus exhibiting simplicity bias at the cost of margin. Further, our claim holds even when the margin in the linear coordinate (minimum distance between linear coordinate of positives and negatives) is significantly smaller than the margin in the slab coordinate. The proof of the following theorem is presented in Appendix~\ref{sec:proof}.

\begin{theorem}\label{thm:main}
	Let $f(x) = \sum_{j=1}^{k} v_j \cdot \textrm{ReLU}(\sum_{i=1}^{d} w_{i,j} x_i)$ 	denote a one-hidden-layer neural network with $k$ hidden units and ReLU activations.
	Set $v_{j} = \pm \sfrac{1}{\sqrt{k}}$ w.p. $\sfrac{1}{2}$ $\forall j \in [k]$.
	Let $\{(x^{i},y^{i})\}^m_{i=1}$ denote i.i.d. samples from \lsn where $m \in [c d^2, d^\alpha / c]$ for some $\alpha > 2$.
	Then, given $d > \Omega(\sqrt{k}\log k)$ and initial $w_{ij} \sim \mathcal{N}(0,\frac{1}{dk \log^4 d})$, after $O(1)$ iterations, mini-batch gradient descent (over $w$) with hinge loss, {step size $\eta = \Omega{{(\log d)^{\sfrac{\minus 1}{2}}}}$}, mini-batch size $\Theta(m)$, satisfies:
	\vspace{-7.5px}
	\begin{itemize}[leftmargin=*]
		\itemsep 0pt
		\topsep 0pt
		\parskip 0pt
		\item Test error is at most $\sfrac{1}{\textrm{poly}(d)}$
		\item The learned weights of hidden units $w_{ij}$ satisfy:
		\small
		$$ \underbrace{\abs{w_{1j}} = {\frac{2}{\sqrt{k}} \left(1-\frac{c}{\sqrt{\log d}}\right) + \order{\frac{1}{\sqrt{dk} \log d}}}}_{{\textstyle \textbf{Linear Coordinate}}} ,\; \underbrace{\abs{w_{2,j}} = {\order{\frac{1}{\sqrt{dk} \log d}}}}_{{\textstyle \textbf{3-Slab Coordinate}}} ,\; \underbrace{\norm{w_{3:d, j}} = \order{\frac{1}{\sqrt{k} \log d}}}_{{\textstyle \textbf{$d\minus 2$ Noise Coordinates}}} $$
		\normalsize
	\end{itemize}
\vspace{-5px}
with probability greater than $1-\frac{1}{\textrm{poly}(d)}$. Note that $c$ is a universal constant.
\end{theorem}
\vspace{-5px}
{\bf Remarks}: First, we see that the trained model essentially relies only on the linear coordinate $w_{1j}$---SGD sets the value of $w_{1j}$ roughly $\tilde{\Omega}(\sqrt{d})$ larger than the slab coordinates $w_{2j}$ that do not change much from their initial value.
Second, the initialization we use is widely studied in the deep learning theory~\cite{mei2018mean,woodworth2020kernel} as it better reflects the practical performance of neural networks~\cite{chizat2019lazy}.
Third, given that \lsn is linearly separable, Brutzkus et al.~\cite{brutzkus2017sgd} also guarantee convergence of test error. However, our result additionally gives a precise description of the \emph{final} classifier. Finally, we note that our result shows that extreme SB bias holds even for overparameterized networks ($k=O(d^2/\log d)$). 

\begin{figure}[t]
	\captionsetup[subfloat]{farskip=2pt,captionskip=2pt}
	\subfloat[original, \SFc-randomized and \SF-randomized logit distribution for true positives]{
		\includegraphics[width=\linewidth]{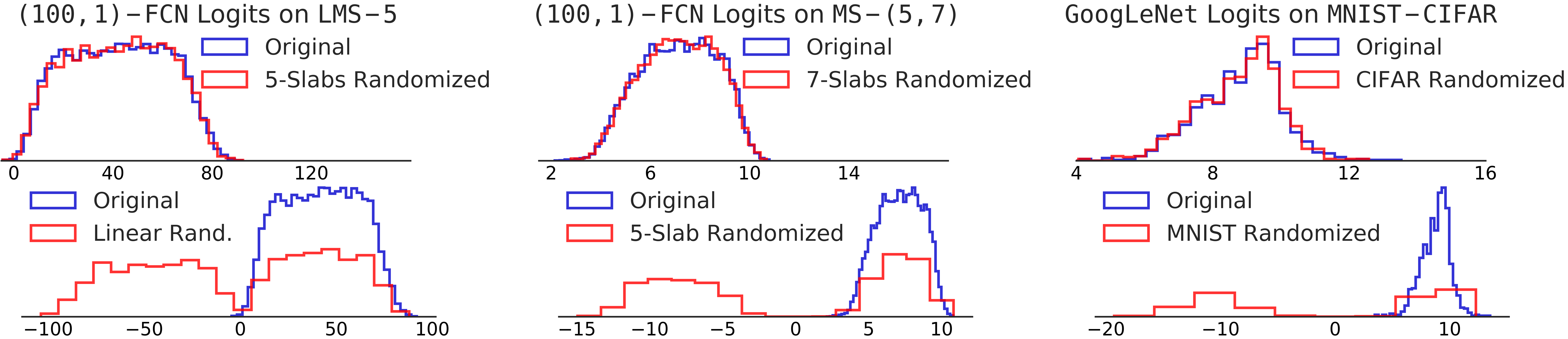}
	}\\
	\subfloat[$\{\mathtt{S},\mathtt{S}^c\}$-Randomized AUCs of (2000,1)-\texttt{FCN}s]{
		\raisebox{0.97\height}{\resizebox{0.48\linewidth}{!}{\def\arraystretch{1.15}
\large
\centering
\begin{tabular}{ccccc}
	\toprule
     \multirow{2}{*}{Dataset} & \multirow{2}{*}{Set $\mathtt{S}^c$} & \multirow{2}{*}{Size $|\mathtt{S}^c|$} &  \multicolumn{2}{c}{Randomized AUC} \\ \cmidrule{4-5}
      & & & Set $\mathtt{S}$ & Set $\mathtt{S^c}$ \\ \midrule
     \multirow{2}{*}{\texttt{LMS-5}}   & \multirow{2}{*}{\texttt{\shortstack{all\\5-Slabs}}} & $49$ & $0.50 \pm 0.01$ & $1.00 \pm 0.00$ \\
     & & $249$ & $0.50 \pm 0.01$ & $1.00 \pm 0.00$ \\ \midrule
     \multirow{2}{*}{\texttt{MS-(5,7)}}   & \multirow{2}{*}{\texttt{\shortstack{all\\7-Slabs}}} & $49$ & $0.49 \pm 0.01$ & $1.00 \pm 0.00$ \\
     & & $249$ & $0.50 \pm 0.01$ & $1.00 \pm 0.00$ \\ \bottomrule     
\end{tabular}
\normalsize}}
	}
	\subfloat[Decision boundaries of trained (100,1)-\texttt{FCNs}]{
		\raisebox{-0.01\height}{\includegraphics[width=0.52\linewidth]{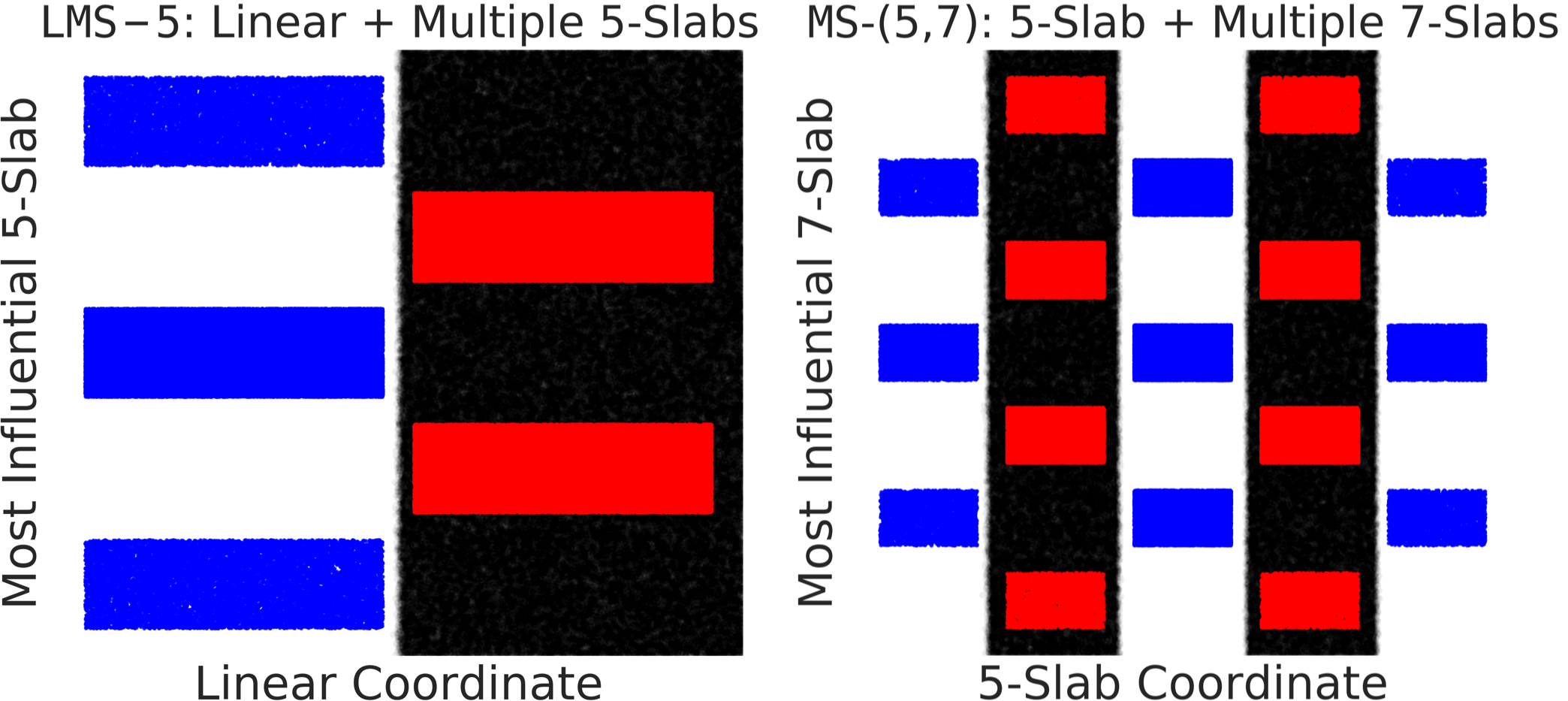}}
	}
	\caption{{Extreme SB on~\lmsfive,~\msfiveseven~and~\mnistcifar~datasets} (a) \SF-randomized logit distribution of true positives is essentially identical to the original logit distribution of true positives (before randomization). However, \SFc-randomized logit distribution of true positives is a randomly shuffled version of original logit distribution; \SFc-randomized logits are \emph{shuffled across true positives and negatives}. (b) \{\SF, \SFc\}-randomized AUCs (summarized in~\Cref{tab:metrics}) are $0.5$ and $1.0$ respectively for varying number of complex features $\abs{\mathtt{S}^c}$. (c) \texttt{FCN} decision boundaries projected onto \SF \& the most influential coordinate in \SFc shows that the boundary depends only on \SF and is invariant to \SFc.}
	\label{fig:oodnew}
\end{figure}
\subsection{Simplicity Bias (SB) is Extreme in Practice}\label{sec:sb-emp}
We now establish the extreme nature of SB on datasets with features of varying simplicity---\lmsfive, \msfiveseven, \mnistcifar (described in Section~\ref{sec:prelims})---across {\em multiple model architectures and optimizers}. 
Recall that (a) the simplicity of  one-dimensional building blocks is defined as the number of pieces required by a piecewise linear classifier acting {\em only} on that block to get optimal accuracy and (b) \lmsfive has one linear block \& multiple $5$-slabs, \msfiveseven has one $5$-slab and multiple $7$-slabs and \mnistcifar concatenates \mnist and \cifar images. 
We now use \SF to denote the simplest feature in each dataset: linear in \lmsfive, 5-slab in \msfiveseven, and \mnist in \mnistcifar.

We first consistently observe that SGD-trained models trained on \lmsfive and \msfiveseven datasets exhibit extreme SB: they \emph{exclusively} rely on the simplest feature \SF and remain invariant to all complex features \SFc. 
Using \SF-randomized \& \SFc-randomized metrics summarized in Table~\ref{tab:metrics}, we first establish extreme SB on fully-connected (\texttt{FCN}), convolutional (\texttt{CNN}) \& sequential (\texttt{GRU}~\cite{cho2014learning}) models.
We observe that the \SF-randomized AUC is 0.5 across models. That is, unsurprisingly, all models are critically dependent on \SF. Surprisingly, however, \SFc-randomized AUC of all models on both datasets equals 1.0. That is, arbitrarily perturbing \SFc coordinates has {\em no impact} on the class predictions or the ranking of true positives' logits against true negatives' logits. One might expect that perturbing \SFc would at least bring the {\em logits} of positives and negatives closer to each other. ~\Cref{fig:oodnew}(a) answers this in negative---the logit distributions over true positives of \texttt{(100,1)-FCN}s (i.e., with width 100 \& depth 1) remain unchanged even after randomizing \emph{all} complex features \SFc. {Conversely, randomizing the simplest feature \SF randomly shuffles the original logits across true positives as well as true negatives.} 
	The two-dimensional projections of \texttt{FCN} decision boundaries in~\Cref{fig:oodnew}(c) visually confirm that \texttt{FCN}s exclusively depend on the simpler coordinate \SF and are invariant to all complex features \SFc. 
	
Note that sample size and model architecture do not present any obstacles in learning complex features \SFc to achieve 100\% accuracy. 
In fact, if \SF is \emph{removed} from the dataset, SGD-trained models with the same sample size indeed rely on \SFc to attain 100\% accuracy. 
Increasing the number of complex features does not mitigate extreme SB either. 
~\Cref{fig:oodnew}(b) shows that even when there are 249 complex features and only one simple feature, \texttt{(2000,1)-FCNs} exclusively rely on the simplest feature \SF; randomizing \SFc keeps AUC score of $1.0$ intact but simply randomizing \SF drops the AUC score to $0.5$.
\texttt{(2000,1)-FCNs} exhibit extreme SB despite their expressive power to learn large-margin classifiers that rely on all simple \emph{and} complex features.

Similarly, on the \mnistcifar dataset, 
MobileNetV2~\cite{sandler2018mobilenetv2}, GoogLeNet~\cite{szegedy2015going}, ResNet50~\cite{he2015deep} and DenseNet121~\cite{huang2016densely} exhibit extreme SB.
All models exclusively latch on to the simpler \texttt{MNIST} block to acheive 100\% accuracy and remain invariant to the \texttt{CIFAR} block, even though the \texttt{CIFAR} block alone is almost fully predictive of its label---GoogLeNet attains 95.4\% accuracy on the corresponding \texttt{CIFAR} binary classification task.
~\Cref{fig:oodnew}(a) shows that randomizing the simpler \texttt{MNIST} block randomly shuffles the logit distribution of true positives whereas randomizing the \texttt{CIFAR} block has no effect---the \texttt{CIFAR}-randomized and original logit distribution over true positives essentially overlap.

To summarize, we use \SF-randomized and \SFc-randomized metrics to establish that models trained on synthetic and image-based datasets exhibit extreme SB: 
\emph{If all features have full predictive power, NNs rely exclusively on the simplest feature \SF and remain invariant to all complex features \SFc}.
We further validate our results on extreme SB across model architectures, activation functions, optimizers and regularization methods such as $\ell_2$ regularization and dropout in~\Cref{sec:ood}.

\subsection{Extreme Simplicity Bias (SB) leads to Non-Robustness}
\label{sec:ood-uap}
Now, we discuss how our findings about extreme SB in~\Cref{sec:sb-emp} can help reconcile poor OOD performance and adversarial vulnerability with superior generalization on the same data distribution.

\textbf{Poor OOD performance}: 
Given that neural networks tend to heavily rely on spurious features~\cite{mccoy-etal-2019-right,oakden2019hidden}, state-of-the-art accuracies on large and diverse validation sets provide a false sense of security; even benign distributional changes to the data (e.g., domain shifts) during prediction time can drastically degrade or even nullify model performance.
This phenomenon, though counter-intuitive, can be easily explained through the lens of extreme SB.
Specifically, we hypothesize that spurious features are {\em simple}.
	This hypothesis, when combined with extreme SB, explains the outsized impact of spurious features.
For example, ~\Cref{fig:oodnew}(b) shows that simply perturbing the simplest (and potentially spurious in practice) feature \SF drops the AUC of trained neural networks to $0.5$, thereby nullifying model performance.
Randomizing \emph{all} complex features \SFc---5-slabs in~\lmsfive, 7-slabs in \msfiveseven, \texttt{CIFAR} block in \mnistcifar---has negligible effect on the trained neural networks---\SFc-randomized and original logits essentially overlap---even though \SFc and \SF have equal predictive power
This further implies that approaches \cite{hendrycks2016baseline,liang2017enhancing} that aim to detect distribution shifts based on model outputs such as logits or softmax probabilities may themselves fail due to extreme SB.

\textbf{Adversarial Vulnerability}: 
Consider a classifier $f^*$  that attains 100\% accuracy on the \lmsfive dataset by taking an average of the linear classifier on the linear coordinate and $d\minus 1$ piecewise linear classifiers, one for every $5$-slab coordinate. By relying on all $d$ features, $f^*$ has ${\mathcal{O}}(\sqrt{d})$ margin. Now, given the large margin, $f^*$ also attains high robust accuracy to $\ell_2$ adversarial perturbations that have norm ${\mathcal{O}}(\sqrt{d})$---the perturbations need to attack at least $\Omega(d)$ coordinates to flip model predictions.   
However, despite high model capacity, SGD-trained NNs do not learn robust and large-margin classifiers such as $f^*$. Instead, due to extreme SB, SGD-trained NNs exclusively rely on the simplest feature \SF. Consequently, $\ell_2$ perturbations with norm $\mathcal{O}(1)$ are enough to flip predictions and degrade model performance. We validate this hypothesis in~\Cref{fig:uap}, where \texttt{FCN}s, \texttt{CNN}s and \texttt{GRU}s trained on \lmsfive and \msfiveseven as well as DenseNet121 trained on \mnistcifar are vulnerable to small  universal (i.e., data-agnostic) adversarial perturbations (UAPs) of the simplest feature \SF. For example,~\Cref{fig:uap} shows that the $\ell_2$ UAP of DenseNet121 on \mnistcifar only attacks a few pixels in the simpler \texttt{MNIST} block and does not perturb the \texttt{CIFAR} block.
Extreme SB also explains why data-agnostic UAPs of one model transfer well to another: the notion of simplicity is consistent across models;~\Cref{fig:uap} shows that \texttt{FCN}s, \texttt{CNN}s and \texttt{GRU}s trained on~\lmsfive and~\msfiveseven essentially learn the same UAP.
Furthermore, invariance to complex features \SFc (e.g., \texttt{CIFAR} block in \mnistcifar) due to extreme SB explains why ``natural"~\cite{hendrycks2019natural} and semantic label-relevant perturbations~\cite{bhattad2020unrestricted} that modify the true image class do not alter model predictions.  
\begin{figure}[t]
	\centering 
	\includegraphics[width=\linewidth]{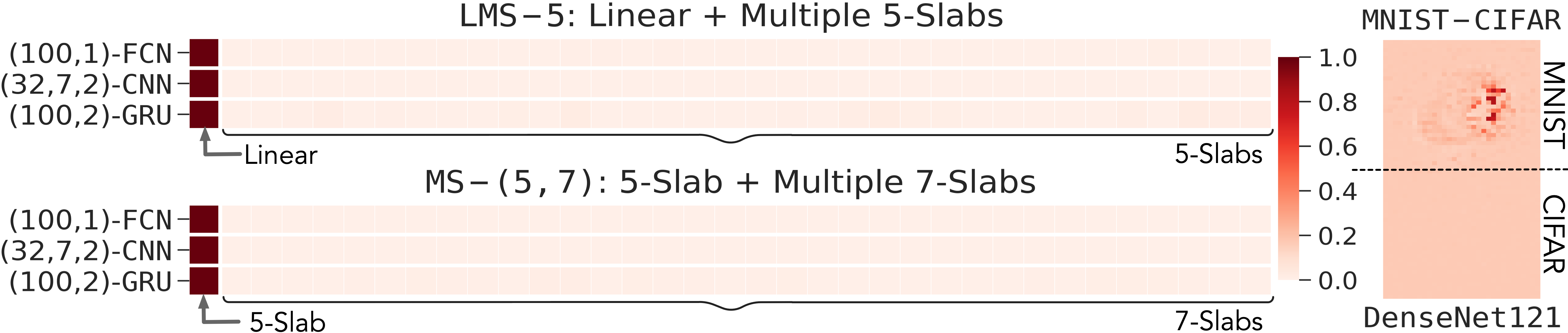}
	\caption{Extreme SB results in vulnerability to small-magnitude model-agnostic and data-agnostic Universal Adversarial Perturbations (UAPs) that nullify model performance by only perturbing the simplest feature \SF. {$\ell_2$ UAPs utilize most of the perturbation budget to attack \SF alone: $99.6\%$ for linear in \lmsfive, $99.9\%$ for $5$-slab in \msfiveseven and $99.3\%$ for \texttt{MNIST} pixels in \mnistcifar. The UAPs in this figure are rescaled for visualization purposes.} 
	}
	\label{fig:uap}
\end{figure}

To summarize, through theoretical analysis and extensive experiments on synthetic and image-based datasets, we (a) establish that SB is extreme in nature across model architectures and datasets and (b) show that extreme SB can result in poor OOD performance and adversarial vulnerability, {\em even when all simple and complex features have equal predictive power}.

\vspace{-5px}
\section{Extreme Simplicity Bias (SB) can hurt Generalization}
\label{sec:gen}
\begin{table}[b]
	\centering
\def\arraystretch{1.1}
\resizebox{1.\linewidth}{!}{%
\begin{tabular}{r|ccccccc}
\toprule
	Accuracy & \texttt{(100,1)-FCN} & \texttt{(200,1)-FCN} & \texttt{(300,1)-FCN} & \texttt{(100,2)-FCN} & \texttt{(200,2)-FCN} & \texttt{(300,2)-FCN} \\ \midrule
	Training Data &  $0.984 \pm 0.003$ &  $0.998 \pm 0.000$ &  $0.995 \pm 0.000$ &  $0.999 \pm 0.000$ &  $0.997 \pm 0.002$ & $0.998 \pm 0.002$ \\
	Test Data & $0.940 \pm 0.002$ & $0.949 \pm 0.003$ & $0.948 \pm 0.002$ & 
		$0.945 \pm 0.003$ & $0.946 \pm 0.003$ & $0.947 \pm 0.002$ \\
	\SFc-Randomized & $0.941 \pm 0.001$ & $0.946 \pm 0.001$ & $0.946 \pm 0.001$ & $0.944 \pm 0.001$ & $0.945 \pm 0.001$ & $0.946 \pm 0.001$ \\
	\SF-Randomized & $0.498 \pm 0.001$ & $0.498 \pm 0.000$ & $0.497 \pm 0.001$ & $0.498 \pm 0.001$ & $0.497 \pm 0.000$ & $0.498 \pm 0.001$ \\ \bottomrule
\end{tabular}%
}

	\vspace{5px}
	\caption{Extreme SB can hurt generalization: \texttt{FCN}s of depth $\{1,2\}$ and width $\{100,200,300\}$ trained on \nlmsseven data with SGD attain $95\%$ test accuracy. The randomized accuracies show that \texttt{FCN}s exclusively rely on the simpler {noisy linear} feature \SF and remain invariant to all 7-slab features that have 100\% predictive power.}
	\label{table:nlms}
\end{table}

\vspace{-5px}
In this section, we show that, contrary to conventional wisdom, extreme SB can potentially result in suboptimal generalization of SGD-trained models on the same data distribution as well. 
This is because exclusive reliance on the simplest feature \SF can persist even when \emph{every complex feature in \SFc has significantly greater predictive power than \SF}.

We verify this phenomenon on \nlmsseven data defined in Section~\ref{sec:prelims}. Recall that \nlmsseven has one noisy linear coordinate \SF with 95\% predictive power (i.e., 10\% noise in linear coordinate) and multiple 7-slab coordinates \SFc, each with 100\% predictive power.
Note that our training sample size is large enough for \texttt{FCN}s of depth $\{1,2\}$ and width $\{100,200,300\}$ trained on \SFc \textit{only} (i.e., after removing \SF from data) to attain 100\% test accuracy.
However, when trained on \nlmsseven (i.e., including \SF), SGD-trained \texttt{FCN}s exhibit extreme SB and \emph{only} rely on \SF, the noisy linear coordinate. 
In~\Cref{table:nlms}, we report accuracies of SGD-trained \texttt{FCN}s that are selected based on validation accuracy after performing a grid search over four SGD hyperparameters: learning rate, batch size, momentum, and weight decay.
The train, test and randomized accuracies in~\Cref{table:nlms} collectively show that \texttt{FCN}s exclusively rely on the noisy linear feature and consequently attain 5\% generalization error. 

To summarize, the mere presence of a simple-but-noisy feature in \nlmsseven data can significantly degrade the performance of SGD-trained \texttt{FCN}s due to extreme SB. 
Note that our results show that even an extensive grid search over SGD hyperparameters does not improve the performance of SGD-trained \texttt{FCN}s on \nlmsseven data but does not necessarily imply that mitigating SB via SGD and its variants is impossible.
We provide additional information about the experiment setup in~\Cref{sec:generror}.

\section{Conclusion and Discussion}
\label{sec:discussion}
We investigated Simplicity Bias (SB) in SGD-trained neural networks (NNs) using synthetic and image-based datasets that (a) incorporate a precise notion of feature simplicity, (b) are amenable to theoretical analysis and (c) capture the non-robustness of NNs observed in practice.
We first showed that one-hidden-layer ReLU NNs provably exhibit SB on the \lsn dataset.
Then, we analyzed the proposed datasets to empirically demonstrate that SB can be extreme, and can help explain poor OOD performance and adversarial vulnerability of NNs.
We also showed that, contrary to conventional wisdom, extreme SB can potentially hurt generalization.

{\bf Can we mitigate SB?}
It is natural to wonder if any modifications to the standard training procedure can help in mitigating extreme SB and its adverse consequences. 
In~\Cref{sec:fixes}, we show that well-studied approaches for improving generalization and adversarial robustness---ensemble methods and adversarial training---do not mitigate SB, at least on the proposed datasets. 
{Specifically, ``vanilla" ensembles of independent models trained on the proposed datasets mitigate SB to some extent by aggregating predictions, but continue to exclusively rely on the simplest feature. That is, the resulting ensemble remains \emph{invariant} to \emph{all} complex features (e.g., $5$-Slabs in \lmsfive). 
Our results suggest that in practice, the improvement in generalization due to vanilla ensembles stem from combining multiple simple-but-noisy features (such as color, texture) and not by learning diverse and complex features (such as shape). 
Similarly, adversarially training \texttt{FCN}s on the proposed datasets increases margin (and hence adversarial robustness) to some extent by combining multiple simple features but does not achieve the maximum possible adversarial robustness;  the resulting adversarially trained models remain invariant to \emph{all} complex features.}

Our results collectively motivate the need for novel algorithmic approaches that avoid the pitfalls of extreme SB.
Furthermore, the proposed datasets capture the key aspects of training neural networks on real world data, while being amenable to theoretical analysis and controlled experiments, and can serve as an effective testbed to understand deep learning phenomena. evaluating new algorithmic approaches aimed at avoiding the pitfalls of SB.

\section*{Broader Impact}
Our work is foundational in nature and seeks to improve our understanding of neural networks.  We do not foresee any significant societal consequences in the short term. However, in the long term, we believe that a concrete understanding of deep learning phenomena is essential to develop reliable deep learning systems for practical applications that have societal impact.

\clearpage
{\small \bibliography{main}}
\bibliographystyle{plain}

\newpage
\appendix
\appendixpage

The supplementary material is organized as follows.
We first discuss additional related work and provide experiment details in~\Cref{sec:related} and~\Cref{sec:setup} respectively.
In~\Cref{sec:ood}, we provide additional experiments to further validate the extreme nature of Simplicity Bias (SB).
Then, in~\Cref{sec:generror}, we provide additional information about the experiment setup used to to show that extreme SB can hurt generalization.
We evaluate the extent to which ensemble methods and adversarial training mitigate Simplicity Bias (SB) in~\Cref{sec:fixes}.
Finally, we provide the proof of Theorem 1 in~\Cref{sec:proof}.

\section{Additional Related Work}
\label{sec:morerelated}
In this section, we provide a more thorough discussion of relevant work related to margin-based generalization bounds, adversarial attacks and robustness, and out-of-distribution (OOD) examples.

{\bf Margin-based generalization bounds}: Building up on the classical work of~\cite{bartlett1998sample}, recent works try to obtain tighter generalization bounds for neural networks in terms of \emph{normalized} margin~\cite{bartlett2017spectrally,neyshabur2017pac,dziugaite2017computing,golowich2017size}. Here, margin is defined as the difference in the probability of the true label and the largest probability of the incorrect labels. While these bounds seem to capture generalization of neural networks at a coarse level, it has been argued~\cite{nagarajan2019uniform} that these approaches may be incapable of fully explaining the generalization ability of neural networks. Furthermore, it is unclear if the notion of model complexity used in these works, based on Lipschitz constant, captures generalization ability accurately. In any case, our results suggest that due to extreme simplicity bias (SB), even if a formulation captures both margin and model complexity accurately, current optimization techniques may not be able to find the optimal solution in terms of generalization \emph{and} robustness-, as they are strongly biased towards small-margin classifiers that exclusively rely on the simplest features.

{\bf Adversarial Defenses}: Neural networks trained using standard procedures such as SGD are extremely vulnerable~\cite{goodfellow2014explaining} to $\epsilon$-bound adversarial attacks such as FGSM~\cite{goodfellow2014explaining}, PGD~\cite{madry2017towards}, CW~\cite{carlini2016towards}, and Momentum~\cite{dong2017boosting}; Unrestricted attacks~\cite{brendel2017decision,engstrom2017exploring} can significantly degrade model performance as well.
Defense strategies based on heuristics such as feature squeezing~\cite{xu2017feature}, denoising~\cite{xie2019feature}, encoding~\cite{buckman2018thermometer}, specialized nonlinearities~\cite{zantedeschi2017efficient} and distillation~\cite{papernot2015distillation} have had limited success against stronger attacks~\cite{athalye2018obfuscated}.  
On the other hand, standard adversarial training~\cite{madry2017towards} and its variants such as~\cite{zhang2019theoretically} are fairly effective on datasets such as \texttt{MNIST}, \texttt{CIFAR-10} and \texttt{CIFAR-100}.
However, on larger datasets such as \texttt{ImageNet}, these methods have limited success~\cite{shafahi2019adversarial}; recent attempts~\cite{wong2020fast,shafahi2019adversarial} that make adversarial training faster do not improve robustness either.
In~\Cref{sec:fixes}, we show that $\ell_2$ adversarial training  on synthetic datasets can improve robustness by some extent but it is unable to learn optimal large-margin $\ell_2$-robust classifiers.

{\bf Detecting OOD Examples}: Neural networks trained using standard training procedures tend to rely on low-level features and spurious correlations and hence exhibit brittleness to benign distributional changes to the data. Recent works thus aim to detect OOD examples using generative models~\cite{ren2019likelihood}, statistical tests \cite{roth2019odds}, and model confidence scores~\cite{hendrycks2016baseline,liang2017enhancing,lakshminarayanan2016simple}. Our experiments in Section 4 that validate extreme SB in practice also show that detectors that directly or indirectly rely on model scores to detect OOD examples may not work well as SGD-trained neural networks can exhibit complete invariance to predictive-but-complex features.

\clearpage
\section{Experiment Details}
\label{sec:setup}
In this section, we provide additional details on the datasets, models, optimization methods and training hyperparameters used in our experiments.

\textbf{One-dimensional Building Blocks}: We first describe the data generation process underlying each building block: linear, noisy linear, and $k$-slab. Then, we introduce a noisy version of the $5$-slab block, which we later use in~\Cref{sec:generror}.
	\begin{itemize}[leftmargin=*]
		\item \underline{Linear$(\gamma, B)$}: The linear block is parameterized by the effective margin $\gamma$ and width $B$. The distribution first samples a label $y \in \{\minus 1, 1\}$ uniformly at random, and then given $y$, $x$ is sampled as follows:
	$ x = y(B\gamma + (B-B\gamma)\cdot \text{U}(0,1), $
	where $\text{U}(0,1)$ is the uniform distribution on $[0,1]$.
		\item \underline{NoisyLinear$(\gamma, B, p)$}: The noisy linear block is parameterized by effective margin $\gamma$, width $B$, and noise parameter $p$. Linear classifiers can attain the optimal classification accuracy of $1-\sfrac{p}{2}$. Given label $y \in \{\minus 1, 1\}$ sampled uniformly at random, $x$ is sampled as follows:
			\begin{align*}
				x = \begin{cases}
 						y(B\gamma + (B-B\gamma)\cdot \text{U}(0,1) & \text{w.p. } p \\
 						\text{U}(\minus \gamma, \gamma) & \text{w.p. } 1\minus p
 					\end{cases}
			\end{align*}
		\item \underline{Slab$(\gamma, B, k)$}: The $k$-slab block is parameterized by effective margin $\gamma$, width $B$, and number of slabs $k$. We use $k \in \{3,5,7\}$ in our paper. The width of each slab, $w_k=\sfrac{2B(1-(k-1)\gamma)}{k}$, in the $k$-slab block is chosen such that the farthest points are at $\minus B$ and $B$. For example, given label $y \in \{\minus 1, 1\}$ and random sign $z \in \{\minus 1, 1\}$ sampled unif. at random, we can sample $x$ from a $3$-slab block as follows:
				\begin{align*}
					x &= \begin{cases}
 							z(\frac{1}{2}w_3 \cdot \text{U}(0,1)) & \text{if } y=\minus 1 \\
 							z(\frac{1}{2}w_3 + 2B\gamma + w_3 \cdot \text{U}(0,1)) & \text{if } y=+1
 						 \end{cases}
				\end{align*}
	For $k$-slab blocks with $k\in \{5,7\}$, the probability of sampling from the two slabs (one on each side) that are farthest away from the origin are $\sfrac{1}{4}$ and $\sfrac{1}{8}$ respectively to ensure that the variance of instances in positive and negative classes, $x_+$ and $x_-$, are equal. 
		\item \underline{NoisySlab$(\gamma, B, k, p)$}: Analogous to the noisy linear block, the noisy variant of the $k$-slab block is additionally parameterized by a noise parameter $p$. In this setting, a $(k\minus 1)$-piecewise linear classifier can attain the optimal classification accuracy of $1-\sfrac{p}{2}$. For example, For example, given label $y \in \{\minus 1, 1\}$ and random sign $z \in \{\minus 1, 1\}$ sampled uniformly at random, we can sample $x$ from a $p$-noisy $3$-slab block as follows:
			\begin{align*}
				x = \begin{cases}
						\begin{cases}
							z(\frac{1}{2}w_3 \cdot \text{U}(0,1)) & \text{if } y=\minus 1 \\ 
							z(\frac{1}{2}w_3 + 2B\gamma + w_3 \cdot \text{U}(0,1)) & \text{if } y=+1 
						\end{cases} & \text{w.p. } 1-p \\ 
						z(\frac{1}{2}w_3 + (2B\gamma - \frac{1}{2}w_3) \cdot \text{U}(0,1)) & \text{w.p. } p
				\end{cases}
			\end{align*}
			
	\end{itemize}

\textbf{Datasets}: We now outline the default hyperparameters for generating the synthetic datasets used in the paper, provide additional details on the \lsn dataset, and introduce two additional synthetic datasets as well as multiple versions of the \mnistcifar dataset (i.e., with different class pairs).
\begin{itemize}[leftmargin=*]
	\item \underline{Synthetic Dataset Hyperparameters}: Recall that we use four $d$-dimensional synthetic datasets---\lmsk, \nlmsk, \msfiveseven, and \msfive---wherein each coordinate corresponds to one of the building blocks described above. Unless mentioned otherwise, for all four datasets, we set the effective margin parameter $\gamma=0.1$, width parameter $B=1$, and noise parameter $p=0.1$ in all blocks/coordinates. Also recall that each dataset comprises at most one ``simple" feature \SF and multiple independent complex features \SFc. In our experiments, all datasets have sample sizes that are large enough for all models considered in the paper to learn complex features \SFc and attain optimal test accuracy, even in the absence of \SF; we use sample sizes of $50000$ for \lmsfive and \msfive and $40000$ for \nlmsseven.
	\item \underline{\lsn Dataset}: Recall that the \lsn dataset (described in Section 3) is a stylized version of the \lmsk that is amenable to theoretical analysis. In \lsn, conditioned on the label $y$, the first and second coordinates of $x$ are \emph{singleton} linear and $3$-slab blocks: linear and $3$-slab blocks have support on $\{\minus 1,1\}$ and $\{\minus 1,0,1\}$ respectively. The remaining coordinates are standard gaussians and not predictive of the label. Each data point $(x_i,y_i) \in \Re^d \times \{-1,1\}$ can be sampled as follows:
			\begin{align*}
				y_i &= \pm 1,\ \mbox{ w.p. } 1/2, \ \ \ \varepsilon_i = \pm 1,\  \mbox{ w.p. } 1/2,\\
				x_{i1} &= y_i & \text{(Linear coordinate)},\\
				x_{i2} &= \Big(\frac{y_i+1}{2}\Big) \varepsilon_i & \text{(Slab coordinate)},\\
				x_{i,3:d} &\sim \mathcal{N}(0, I_{d-2}) & \text{($d \minus 2$ Noise coordinates)}.
			\end{align*}
	\item \underline{Additional Datasets}: We now introduce \nmsfiveseven, the noisy version of \msfiveseven, and three \mnistcifar datasets, each with different \mnist and \cifar classes. 
		\begin{itemize}[leftmargin=*]
			\item \underline{\nmsfiveseven}: Noisy 5-slab and multiple noiseless 7-slab blocks; the first coordinate is a noisy 5-slab block and the remaining $d\minus 1$ coordinates are independent $7$-slab blocks. Note that this dataset comprises a noisy-but-simpler $5$-slab block and multiple noiseless $7$-slab blocks; a $6$-piecewise linear classifier can attain 100\% accuracy by learn any $7$-slab block. 
			\item \underline{\mnistcifar datasets}: Recall that images in the \mnistcifar datasets are concatenations of \mnist and \cifar images. We introduce additional variants of the \mnistcifar using different class pairs to show that our results in the paper are robust to the exact choice of pairs:
				\begin{table}[h]
					\def\arraystretch{.95}
					\centering
					\begin{tabular}{rccccc}
						\toprule
						\multirow{2}{*}{Datasets} & \multicolumn{2}{c}{Class $\minus 1$} & & \multicolumn{2}{c}{Class $+1$}  \\ \cmidrule{2-3} \cmidrule{5-6}  
						& \mnist & \cifar & & \mnist & \cifar \\ \midrule
						\mnistcifara & Digit $0$ & Automobile & & Digit $1$ & Truck \\
						\mnistcifarb & Digit $1$ & Automobile & & Digit $4$ & Truck \\
						\mnistcifarc & Digit $0$ & Airplane & & Digit $1$ & Ship \\ \bottomrule
					\end{tabular}
					\vspace{5px}
					\caption{Three \mnistcifar datasets. We use \mnistcifara in the paper. In \mnistcifarb, we use different \mnist classes: digits 1 and 4. In \mnistcifarc, we use different \cifar classes: airplane and ship. Our results in Section 4 hold on all three \mnistcifar datasets.}
					\label{table:mnistcifar}
				\end{table}

		\end{itemize}
\end{itemize}

\vspace{-15px}
\textbf{Models}: Here, we briefly describe the models (and its abbreviations) used in the paper. We use fully-connected (\texttt{FCN}s), convolutional (\texttt{CNN}s), and sequential neural networks (\texttt{GRU}s \cite{cho2014learning}) on synthetic datasets. Abbreviations $(w,d)$-\texttt{FCN} denotes \texttt{FCN} with width $w$ and depth $d$, $(f,k,d)$-\texttt{CNN} denotes $d$-layer \texttt{CNN}s with $f$ filters of size $k \times k$ in each layer with and $(h,l,d)$-\texttt{GRU} denotes $d$-layer $d$-layer \texttt{GRU} with input dimensionality $l$ and hidden state dimensionality $h$. On \mnistcifar, we train MobileNetV2~\cite{sandler2018mobilenetv2}, GoogLeNet~\cite{szegedy2015going}, ResNet50~\cite{he2015deep} and DenseNet121~\cite{huang2016densely}. 

\textbf{Training Procedures}:  Unless mentioned otherwise, we use the following hyperparameters for standard training and adversarial training on synthetic and \mnistcifar data:
\begin{itemize}[leftmargin=*]
	\item \underline{Standard Training}: On synthetic datasets, we use Stochastic Gradient Descent (SGD) with (fixed) learning rate $0.1$ and batch size $256$, and $\ell_2$ regularization $5 \cdot 10^{-7}$. On \mnistcifar datasets, we use SGD with initial learning rate $0.05$ with decay factor of $0.2$ every $30$ epochs, momentum $0.9$ and $\ell_2$ regularization $5\cdot 10^{-5 }$. We do not use data augmentation. We run all models for at most 500 epochs and stop early if the training loss goes below $10^{-2}$. 
	\item \underline{Adversarial Training}: We use the same SGD hyperparameters (as described above) on synthetic and \mnistcifar datasets. We use Projected Gradient Descent (PGD) Adversarial Training \cite{madry2017towards} to adversarially train models. We use learning rate $0.1$ and $40$ iterations to generate $\ell_2$ \& $\ell_{\infty}$ perturbations
\end{itemize}

\clearpage
\section{Additional Results on the Extreme Nature of Simplicity Bias (SB)}
\label{sec:ood}
Recall that Section 4 of the paper establishes the extreme nature of SB: \emph{If all features have full predictive power, NNs rely exclusively on the simplest feature \SF and remain invariant to all complex features \SFc}---in Section 4 of the paper.  Now, we further validate the extreme nature of SB across model architectures, datasets, optimizers, activation functions and regularization. We also analyze the effect of input dimensionality, number of complex features, choice and scaling of random initialization and non-random initialization.

\subsection{Effect of Model Architecture}

In this section, we supplement our results in Section 4 of the paper by showing that extreme simplicity bias (SB) persists across several model architectures and on synthetic as well as image-based datasets.
In~\Cref{table:ood_synthetic}, we present \{\SF,\SFc\}-Randomized AUCs for FCNs, CNNs and GRUs with depth \{1,2\} trained on \lms and \msfiveseven datasets and state-of-the-art CNNs trained on \mnistcifara. 
While the \SFc-randomized AUC equals $1.00$ (perfect classification), we see that the \SF-randomized AUCs are approximately $0.5$ for all models. 
This is because all models essentially only rely on the simplest feature \SF and remain invariant to all complex features \SFc, even though all features have equal predictive power.

\begin{table}[h]
	\centering 
	\def\arraystretch{1.1}
\scriptsize
\begin{tabular}{ccclcc}
		  \toprule
      \multirow{2}{*}{Dataset} & \multirow{2}{*}{Set $\mathtt{S}$} &  \multirow{2}{*}{Set $\mathtt{S}^c$} & \multirow{2}{*}{Model} & \multicolumn{2}{c}{Randomized AUC} \\ \cmidrule{5-6}
       & & & & Set $\mathtt{S}$ & Set $\mathtt{S^c}$ \\ \midrule
      \multirow{6}{*}{\texttt{LMS-5}}   & \multirow{6}{*}{\texttt{Linear}} & \multirow{6}{*}{\texttt{5-Slabs}} & {\texttt{(100,1)-FCN}} & $0.50$ & $1.00$  \\
        & &  & {\texttt{(100,2)-FCN}} & $0.49$ & $1.00$ \\
        & &  & {\texttt{(32,7,1)-CNN}} & $0.50$ & $1.00$  \\
        & &  & {\texttt{(32,7,2)-CNN}} & $0.50$ & $1.00$  \\
        & &  & {\texttt{(100,10,1)-GRU}} & $0.51$ & $1.00$  \\
        & &  & {\texttt{(100,10,2)-GRU}} & $0.50$ & $1.00$  \\ \cmidrule{1-6}
      \multirow{6}{*}{\texttt{MS-(5,7)}}   & \multirow{6}{*}{\texttt{5-Slab}} & \multirow{6}{*}{\texttt{7-Slabs}}  & {\texttt{(100,1)-FCN}} & $0.50$ & $1.00$  \\
        & &  & {\texttt{(100,1)-FCN}} & $0.50$ & $1.00$ \\
        & &  & {\texttt{(32,7,1)-CNN}} & $0.50$ & $1.00$  \\
        & &  & {\texttt{(32,7,2)-CNN}} & $0.50$ & $1.00$  \\
        & &  & {\texttt{(100,10,1)-GRU}} & $0.50$ & $1.00$  \\
        & &  & {\texttt{(100,10,2)-GRU}} & $0.50$ & $1.00$  \\  \cmidrule{1-6}
      \multirow{3}{*}{\mnistcifara}   & \multirow{4}{*}{\shortstack{\texttt{MNIST}\\block}} & \multirow{4}{*}{\shortstack{\texttt{CIFAR}\\block}} & {\texttt{MobileNetV2}} & $0.52$ & $1.00$ \\
        & &  & {\texttt{GoogLeNet}} & $0.51$ & $1.00$ \\
        & &  & {\texttt{ResNet50}} & $0.50$ & $1.00$ \\ 
        & &  & {\texttt{DenseNet121}} & $0.52$ & $1.00$ \\  \bottomrule
        
  \end{tabular}	
	\vspace{5px}
	\caption{Extreme SB across models trained on synthetic and image-based datasets show that all models exclusively rely on the simplest feature \SF and remain completely invariant to all complex features \SFc}
	\label{table:ood_synthetic}
\end{table}

\vspace{-17px}
\subsection{Effect of \mnistcifar Class Pairs}
\label{sec:actopt}

In this section, we supplement our results on \mnistcifar (in Section 4) in
order to show that extreme SB observed in  
MobileNetV2~\cite{sandler2018mobilenetv2}, GoogLeNet~\cite{szegedy2015going}, ResNet50~\cite{he2015deep} and DenseNet121~\cite{huang2016densely}
does not depend on the exact choice of \mnist and \cifar class pairs used to construct the \mnistcifar datasets.
To do so, we evaluate the \mnist-randomized and \cifar-randomized metrics of the aforementioned models on three datasets--\mnistcifara, \mnistcifarb, \mnistcifarc---described in~\Cref{sec:setup}.
\begin{table}[h]
	\def\arraystretch{1.01}
	\centering 
		\resizebox{\textwidth}{!}{
		\begin{tabular}{rccccccccccc}
		\toprule
		\multirow{2}{*}{Model} & \multicolumn{3}{c}{\mnistcifara AUCs} &  & \multicolumn{3}{c}{\mnistcifarb AUCs} &  & \multicolumn{3}{c}{\mnistcifarc AUCs} \\ \cline{2-4} \cmidrule{6-8} \cmidrule{10-12}
		                       & Standard  & \shortstack{\cifar\\Randomized}  & \shortstack{\mnist\\Randomized}           &  & Standard   & \shortstack{\cifar\\Randomized}          & \shortstack{\mnist\\Randomized}            &  & Standard    & \shortstack{\cifar\\Randomized}         & \shortstack{\mnist\\Randomized}            \\ \midrule
		       
\texttt{MobileNetV2} 
	& $1.00 \pm 0.00$ & $1.00 \pm 0.00$ & $0.53 \pm 0.01$  
	& & $1.00 \pm 0.00$ & $1.00 \pm 0.00$ & $0.53 \pm 0.02$
	& & $1.00 \pm 0.00$ & $1.00 \pm 0.00$ & $0.50 \pm 0.01$ \\
\texttt{GoogLeNet} 
	& $1.00 \pm 0.00$ & $1.00 \pm 0.00$ & $0.52 \pm 0.02$ 
	& & $1.00 \pm 0.00$ & $1.00 \pm 0.00$ & $0.50 \pm 0.01$
	& & $1.00 \pm 0.00$ & $1.00 \pm 0.00$ & $0.53 \pm 0.01$ \\
\texttt{ResNet50} 
	& $1.00 \pm 0.00$ & $1.00 \pm 0.00$ & $0.50 \pm 0.01$
	& & $1.00 \pm 0.00$ & $1.00 \pm 0.00$ & $0.51 \pm 0.01$
	& &  $1.00 \pm 0.00$ & $1.00 \pm 0.00$ & $0.50 \pm 0.03$ \\
\texttt{DenseNet121} 
	& $1.00 \pm 0.00$ & $1.00 \pm 0.00$ & $0.53 \pm 0.02$
	& & $1.00 \pm 0.00$ & $1.00 \pm 0.00$ & $0.52 \pm 0.01$
	& & $1.00 \pm 0.00$ & $1.00 \pm 0.00$ & $0.54 \pm 0.01$ \\
\bottomrule
		\end{tabular}
	}	
	\vspace{5px}
		\caption{(Extreme SB in three \mnistcifar datasets)
	Standard and randomized AUCs of four state-of-the-art CNNs trained on three \mnistcifar datasets.
	The AUC values collectively indicate that all models exclusively rely on the \mnist block.}
	\label{table:ood_mnistcifar}
\end{table}

\Cref{table:ood_mnistcifar} presents the standard, \mnist-randomized and \cifar-randomized AUC values of MobileNetV2, GoogLeNet, ResNet50 and DenseNet121 on three \mnistcifar datasets. We observe that randomizing over the simpler \mnist block is sufficient to fully degrade the predictive power of all models; for instance,
randomizing the \mnist block drops the AUC values of \texttt{ResNet50} from $1.0$ to $0.5$ (i.e., equivalent to random classifier). However, randomizing the \cifar block has no effect---standard AUC and \cifar-randomized AUCs equal $1.0$. 
In contrast, an ideal classifier that relies on \mnist \& \cifar would attain non-trivial AUC even when the \mnist block is randomized.

\subsection{Effect of Optimizers and Activation Functions}
Now, we study the effect of activation function and optimizer on extreme SB. That is, can the usage of different activation functions and optimizer encourage trained neural networks to rely on complex features \SFc in addition to the simplest feature \SF?
\begin{table}[h]
  \def\arraystretch{1.05}
  \centering
  \resizebox{\textwidth}{!}{\begin{tabular}{cccccccc}
  \toprule
  \multirow{2}{*}{\begin{tabular}[c]{@{}c@{}}Activation\\ Function\end{tabular}} & \multicolumn{3}{c}{\lmsseven}         &  & \multicolumn{3}{c}{\msfiveseven} \\ \cmidrule{2-4} \cmidrule{6-8}
& SGD & Adam & RMSProp &  & SGD & Adam & RMSProp \\ \midrule
  ReLU        & $0.499 \pm 0.001$  & $0.497 \pm 0.003$ &  $0.502 \pm 0.004$       & & $0.499 \pm 0.003$ & $0.499 \pm 0.004$ & $0.496 \pm 0.004$ \\
  Leaky ReLU  & $0.501 \pm 0.001$  & $0.497 \pm 0.003$ &  $0.501 \pm 0.005$       & & $0.499 \pm 0.005$ & $0.498 \pm 0.002$       & $0.498 \pm 0.005$ \\
  PReLU        & $0.500 \pm 0.004$  & $0.500 \pm 0.003$ &  $0.501 \pm 0.004$       & & $0.501 \pm 0.004$       & $0.496 \pm 0.003$       & $0.499 \pm 0.002$ \\
  Tanh        & $0.495 \pm 0.001$  & $0.502 \pm 0.004$ &  $0.495 \pm 0.004$       & & $0.498 \pm 0.004$       & $0.499 \pm 0.004$       & $0.498 \pm 0.002$ \\ \bottomrule
\end{tabular}
}

  \caption{(Effect of activation function and optimizers) $(100,2)$-\texttt{FCN}s with multiple activation functions---ReLU, Leaky ReLU \cite{maas2013rectifier}, PReLU \cite{he2015delving}, and Tanh---trained on \lmsfive data using common first-order optimization methods---SGD, Adam~\cite{kingma2014adam}, and RMSProp \cite{tieleman2012lecture}---exhibit extreme SB.}
	\label{table:actopt}
\end{table}

\vspace{-18px}
\Cref{table:actopt} presents the \SF-randomized AUCs of $(100,2)$-FCNs with multiple activation functions---ReLU, Leaky ReLU \cite{maas2013rectifier}, PReLU \cite{he2015delving}, and Tanh---trained on \lmsseven and \msfiveseven datasets using multiple commonly-used optimizers: SGD, Adam~\cite{kingma2014adam},and RMSProp \cite{tieleman2012lecture}. We observe that for all combinations of activations and optimizers, trained FCNs still only rely on simplest feature \SF; \SF-randomized and \SFc-randomized AUCs are approximately $0.50$ and $1.0$ respectively for all optimizers and activation functions. Therefore, in addition to SGD, commonly used first-order optimization methods such as Adam and RMSProp cannot jointly learn large-margin classifiers that rely on learn slab-structured features in the presence of a noisy linear structure.
To summarize, the experiment in \Cref{sec:actopt} shows that simply altering the choice of optimizer and activation function does not have any effect on extreme SB. Similar to the experiments in Section 4 of the paper, all models exclusively rely on simplest feature \SF and remain invariant to complex features \SFc.

\subsection{Effect of $\ell_2$ Regularization and Dropout}
In this section, we use SGD-trained \texttt{FCN}s trained on \lmsseven data to examine the extent to which Dropout \cite{srivastava2013improving} and $\ell_2$ regularization alters the extreme nature of SB.
Specifically, we use Dropout probability parameter $\{0.0, 0.05, 0.10\}$ and $\ell_2$ regularization parameters $\{0.01, 0.001\}$
	when training \texttt{FCN}s with width $100$ and depth $\{1,2\}$ on \lmsseven data using SGD. 
In~\Cref{table:dropout}, we show the standard and \SFc-randomized AUCs equal $1.00$ (perfect classification), whereas the \SF-randomized AUCs are approximately $0.5$ for all models. 
Applying Dropout while reducing the amount of $\ell_2$ regularization has negligible effect on the extreme nature of SB observed in the synthetic or image-based datasets.
\begin{table}[b]
	\def\arraystretch{1.1}
	\centering
	\resizebox{1.0\textwidth}{!}{
    \def\arraystretch{1.1}
    \begin{tabular}{ccccccccccc}
    \toprule
    \multirow{2}{*}{Model} & \multirow{2}{*}{Dropout} & \multicolumn{2}{c}{Standard AUC} & & \multicolumn{2}{c}{\SF-Randomized AUC} & & \multicolumn{2}{c}{\SFc-Randomized AUC} \\ \cmidrule{3-4} \cmidrule{6-7} \cmidrule{9-10}
    & & $\lambda=10^{-2}$ & $\lambda=10^{-4}$ & &
        $\lambda=10^{-2}$ & $\lambda=10^{-4}$ & &
        $\lambda=10^{-2}$ & $\lambda=10^{-4}$ \\ \midrule
\multirow{3}{*}{\texttt{(100,1)-FCN}}
    & $0.00$  & $1.00 \pm 0.00$  &  $1.00 \pm 0.00$  & &
               $0.50 \pm 0.00$  &  $0.50 \pm 0.01$  & &
               $1.00 \pm 0.00$  &  $1.00 \pm 0.00$ \\
    & $0.05$  & $1.00 \pm 0.00$  &  $1.00 \pm 0.00$  & &
               $0.50 \pm 0.00$  &  $0.50 \pm 0.00$  & &
               $1.00 \pm 0.00$  &  $1.00 \pm 0.00$  \\
    & $0.10$  & $1.00 \pm 0.00$  &  $1.00 \pm 0.00$  & &
               $0.50 \pm 0.00$  &  $0.50 \pm 0.00$  & &
               $1.00 \pm 0.00$  &  $1.00 \pm 0.00$   \\
               \cmidrule{1-10}
\multirow{3}{*}{\texttt{(100,2)-FCN}}
    & $0.00$  & $1.00 \pm 0.00$  &  $1.00 \pm 0.00$  & &
               $0.50 \pm 0.00$  &  $0.50 \pm 0.00$  & &
               $1.00 \pm 0.00$  &  $1.00 \pm 0.00$   \\
    & $0.05$  & $1.00 \pm 0.00$  &  $1.00 \pm 0.00$  & &
               $0.50 \pm 0.00$  &  $0.50 \pm 0.00$  & &
               $1.00 \pm 0.00$  &  $1.00 \pm 0.00$   \\
    & $0.10$  & $1.00 \pm 0.00$  &  $1.00 \pm 0.00$  & &
               $0.50 \pm 0.00$  &  $0.50 \pm 0.00$  & &
               $1.00 \pm 0.00$  &  $1.00 \pm 0.00$  \\
    \bottomrule
    \end{tabular}
}

	\vspace{5px}
	\caption{Dropout and $\ell_2$ regularization have no effect on extreme SB of \texttt{FCN}s trained on \lmsseven datasets. The standard and \{\SF,\SFc\}-randomized AUC values of \texttt{(100,1)-FCN}s and \texttt{(100,2)-FCN}s collectively indicate that the models still exclusively latch on to \SF (linear block) and remain invariant to \SFc (7-slab blocks). }
	\label{table:dropout}
\end{table}

\subsection{Effect of Input Dimension and Number of Complex Features}
In this section, we evaluate the performance of \texttt{FCN}s trained on \lmsseven data using SGD to show the extreme SB persists in the low-dimensional setting ($d < 10$) and also with varying number of $7$-slab features $1 \leq |\mathtt{S}^{c}| \leq d$. 
\begin{table}[t]
	\def\arraystretch{1.1}
	\centering
	\resizebox{1.0\textwidth}{!}{
    \def\arraystretch{1.1}
    \begin{tabular}{ccccc}
    \toprule
	Input Dimension $d$ & Number of $7$-Slabs $|\mathtt{S}^c|$ & Standard AUC & \SF-Randomized AUC & \SFc-Randomized AUC \\ \midrule
	$2$ & $1$ & $1.00 \pm 0.00$ & $0.51 \pm 0.02$ & $1.00 \pm 0.00$ \\ \cmidrule{1-5}
	\multirow{2}{*}{$5$} 
		& $1$ & $1.00 \pm 0.00$ & $0.50 \pm 0.00$ & $1.00 \pm 0.00$ \\
		& $3$ & $1.00 \pm 0.00$ & $0.50 \pm 0.00$ & $1.00 \pm 0.00$ \\ \cmidrule{1-5}
	\multirow{3}{*}{$10$} 
		& $1$ & $1.00 \pm 0.00$ & $0.50 \pm 0.00$ & $1.00 \pm 0.00$ \\
		& $3$ & $1.00 \pm 0.00$ & $0.50 \pm 0.01$ & $1.00 \pm 0.00$ \\
		& $6$ & $1.00 \pm 0.00$ & $0.50 \pm 0.00$ & $1.00 \pm 0.00$ \\		 
	\bottomrule
	\end{tabular}
}

	\vspace{5px}
	\caption{Effect of input dimension and number of complex features on extreme SB of \texttt{FCN}s trained on \lmsseven data. The standard and \{\SF,\SFc\}-randomized AUCs of \texttt{(100,1)-FCN}s collectively indicate that the models exclusively latch on to \SF (linear) and remain invariant to \SFc (7-slabs).}
	\label{table:dimnum}
\end{table}
As shown in~\Cref{table:dimnum}, decreasing the input dimension $d$ or the number of complex features $|\mathtt{S}^{c}|$ has no effect on extreme SB of \texttt{FCN}s trained on the \lmsseven dataset. 
Similar to our results in~\Cref{fig:oodnew}, the standard and randomized AUCs collectively show that the SGD-trained \texttt{(100,1)-FCN}s exclusively rely on the linear component and do not rely on the $7$-slab coordinates.

\subsection{Effect of Random Initialization Scale}
Now, we analyze the effect of the choice and scale (i.e., magnitude of the weights of randomly initialized \texttt{FCN}s) of random initialization on simplicity bias using \texttt{FCN}s trained on \lmsseven data.
\begin{table}
	\def\arraystretch{1.1}
	\centering
	\resizebox{0.85\textwidth}{!}{
    \def\arraystretch{1.1}
    \begin{tabular}{ccccc}
    \toprule
	Initialization & Scaling Factor & Standard AUC & \SF-Randomized AUC & \SFc-Randomized AUC \\ \midrule
	\multirow{4}{*}{Kaiming} 
		& 0.1 & $1.00 \pm 0.00$ & $0.50 \pm 0.00$ & $1.00 \pm 0.00$ \\ 
		& 0.5 & $1.00 \pm 0.00$ & $0.49 \pm 0.01$ & $1.00 \pm 0.00$ \\ 
		& 2.0 & $1.00 \pm 0.00$ & $0.50 \pm 0.00$ & $1.00 \pm 0.00$ \\ 
		& 10.0 & $1.00 \pm 0.00$ & $0.50 \pm 0.00$ & $1.00 \pm 0.00$ \\ \cmidrule{1-5} 
	\multirow{4}{*}{Xavier} 
		& 0.1 & $1.00 \pm 0.00$ & $0.50 \pm 0.00$ & $1.00 \pm 0.00$ \\ 
		& 0.5 & $1.00 \pm 0.00$ & $0.50 \pm 0.00$ & $1.00 \pm 0.00$ \\ 
		& 2.0 & $1.00 \pm 0.00$ & $0.49 \pm 0.01$ & $1.00 \pm 0.00$ \\ 
		& 10.0 & $1.00 \pm 0.00$ & $0.51 \pm 0.01$ & $1.00 \pm 0.00$ \\ 
	\bottomrule
	\end{tabular}
}

	\vspace{5px}
	\caption{Effect of choice (Kaiming and Xavier) and scale of random initialization on simplicity bias of \texttt{FCN}s trained on \lmsseven data. The AUCs of \texttt{(100,1)-FCN}s collectively show increasing the scale of random initialization does not alleviate extreme simplicity bias in this setting.}
	\label{table:randinit}
\end{table}

As shown in~\Cref{table:randinit}, the choice (Kaiming and Xavier) and the scale of random initialization do not alter the extreme SB phenomenon on the \lmsseven dataset. That is, scaling the randomly initialized models by up to $0.1$ and $10.0$ has no effect on simplicity bias---SGD-trained \texttt{(100,1)-FCN}s exclusively rely on the linear component and do not rely on the $7$-slab coordinates.

\subsection{Effect of Non-random Initialization}
In this section, we investigate the effect of non-random initialization on simplicity bias using \texttt{FCN}s trained on \msseven and \lmsseven data. 
The goal of this experiment is to determine the extent to which simplicity bias persists when the untrained network (at timestep $t=0$) attains non-random standard accuracy by relying on one or more ``complex" $7$-slab features.

To vary the degree of non-random initialization $\alpha$, we obtain model $\mathtt{M}_{\alpha}$ by linearly interpolating the weights of a randomly initialized network $\mathtt{M}_{\text{rand}}$ and a network $\mathtt{M}_{\text{slab}}$ that exclusively relies on one or more ``complex" $7$-slab features to attain $100\%$ accuracy on the \lmsseven dataset. That is, $\mathtt{M}_{\alpha} \equiv \alpha \cdot \mathtt{M}_{\text{slab}} + (1-\alpha) \cdot \mathtt{M}_{\text{rand}}$. Note that $\mathtt{M}_{\text{slab}}$ is trained on \msseven data to attain $100\%$ standard accuracy by relying on one or more $7$-slab coordinates. As shown in~\Cref{fig:interp_acc}(a), increasing the interpolation constant $\alpha$ monotonically increases the standard accuracy of $\mathtt{M}_{\alpha}$ on \msseven data. 

Now, we use the linearly interpolated model $\mathtt{M}_{\alpha}$ (for varying values of $\alpha$) as initialization and train $\mathtt{M}_{\alpha}$ on \lmsseven data, which additionally consists of a ``simple" linearly-separable coordinate.
To maintain the input dimensionality, we obtain \lmsseven data by replacing a $7$-slab coordinate by the simpler linear coordinate. 
In order to maintain the non-random accuracy of $\mathtt{M}_{\alpha}$, we use coordinate-randomized AUCs to choose and replace a $7$-slab coordinate that the model does not depend on. 
	
\begin{figure}
	\subfloat[]{\includegraphics[width=0.33\linewidth]{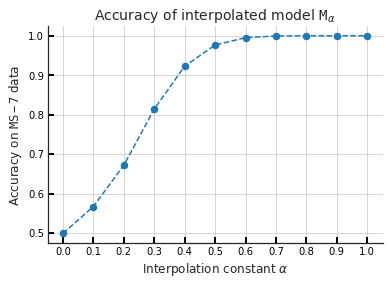}}
	\subfloat[]{\includegraphics[width=0.33\linewidth]{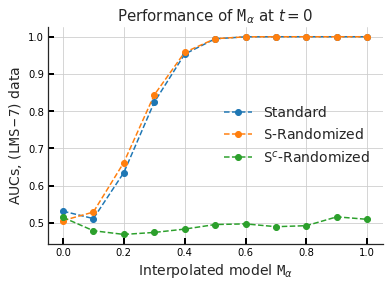}}
	\subfloat[]{\includegraphics[width=0.33\linewidth]{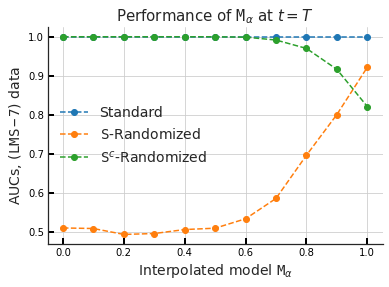}}
	\caption{
	Effect of non-random initialization on simplicity bias. Subplot (a) standard test accuracy (on \msseven data) of linearly-interpolated model $\mathtt{M}_{\alpha} \equiv \alpha \cdot \mathtt{M}_{\text{slab}} + (1-\alpha) \cdot \mathtt{M}_{\text{rand}}$ increases monotonically with the interpolation constant $\alpha$. Subplots (b) and (c) show how standard, \SF-randomized and \SFc-randomized AUCs vary with interpolated model $\mathtt{M}_{\alpha}$ before and after training on \lmsseven data. 
	}
	\label{fig:interp_acc}
\end{figure}

As expected,~\Cref{fig:interp_acc}(b) shows that prior to training on \lmsseven data, the interpolated models $\mathtt{M}_{\alpha}$ attain \SFc-randomized AUC $0.50$ and \SFc-randomized accuracy equals standard AUC because the models exclusively rely on one or more $7$-slab coordinates. However, as shown in~\Cref{fig:interp_acc}(c), after training $\mathtt{M}_{\alpha}$ on \lmsseven data, the models no longer exhibit exclusive reliance on $7$-slab features. In particular, when $\alpha \leq 0.5$, the models now exclusively rely on the linear coordinate only. When $\alpha \geq 0.5$, the dependence on the $7$-slab coordinates is considerably reduced. Surprisingly, even when $\alpha=1.0$, we observe that the model additionally relies on the linear coordinate, possibly due to non-zero training loss at initialization. These results collectively suggest that in this setting, non-random initialization by first training the model on complex features only can be effective in mitigating extreme simplicity bias to some extent.

\clearpage
\section{Additional Results on the Effect of Extreme SB on Generalization}
\label{sec:generror}
Recall that in Section 5 of the paper,  we showed that extreme SB can result in suboptimal generalization of SGD-trained models on the same data distribution.
In this section, we present additional information about the experiment setup used in Section 5 to show that SB can worsen standard generalization.

\subsection{Experiment Setup in Section 5}
\label{subsec:geninfo}

We now provide additional information about the experimental setup used in~\Cref{sec:gen} of the paper, where we show that extreme simplicity bias can result in suboptimal generalization.
We train fully-connected networks (\texttt{FCN}s) of width $\{100,200,300\}$ and depth $\{1,2\}$ using SGD on $50$-dimensional \nlmsseven dataset of $40000$ samples, which comprises of a noisy linear coordinate ($10\%$ noise) and $49$ $7$-slab coordinates. 
For each model architecture, we perform a grid search over SGD hyperparameters---learning rate, batch size, momentum, and weight decay---and report standard and randomized test accuracies of the model (in~\Cref{table:nlms}) that perform best on a \nlmsseven validation dataset.
We perform a grid search over the following SGD hyperparameters:
\begin{itemize}[leftmargin=*]
	\item Learning rate: $\{0.001, 0.01, 0.05, 0.1, 0.3\}$
	\item Batch size: $\{4, 16, 64, 256\}$
	\item Weight decay: $\{0, 0.00005, 0.0005\}$
	\item Momentum: $\{0, 0.9, 0.95\}$
\end{itemize}
We train all models for at most $10$ million updates with constant learning rate schedule and stop early if the training loss diverges or goes below $0.01$. In this setting, $10$ million updates is equivalent to $1000$ epochs with batch size $4$ and $64000$ epochs with batch size $256$.
As mentioned in~\Cref{sec:gen}, the training sample size of $40000$ data points is large enough for \texttt{FCN}s of depth $\{1,2\}$ and width $\{100,200,300\}$ trained on \msseven data (i.e., \SFc \textit{only}, after removing \SF from data) to attain $\approx 100\%$ test accuracy using SGD with \textit{learning rate $0.3$, batch size $256$, weight decay $0.0005$, and momentum $0.9$}. 
The optimal hyperparameters for \texttt{FCN}s trained on \nlmsseven data are provided in~\Cref{table:hparams}. 
Note that we do not consider $\ell_2$ weight decay values larger than $0.0005$ because it results in $95\%$ test and train accuracy even after $10$ million updates by preventing \texttt{FCN}s from overfitting to the noise in the linear component.
Also note that \texttt{(100,1)-FCN}s are not able to completely overfit due to insufficient representation capacity when trained with the chosen SGD hyperparameters (see~\Cref{table:hparams}).

\begin{table}[h]
	\def\arraystretch{1.1}
	\centering
	\centering
\def\arraystretch{1.1}
\resizebox{1.\linewidth}{!}{%
\begin{tabular}{r|ccccccc}
\toprule
	Hyperparameter & \texttt{(100,1)-FCN} & \texttt{(200,1)-FCN} & \texttt{(300,1)-FCN} & \texttt{(100,2)-FCN} & \texttt{(200,2)-FCN} & \texttt{(300,2)-FCN} \\ \midrule
	Learning rate &  $0.30$ &  $0.10$ &  $0.01$ &  $0.30$ &  $0.10$ & $1.00$ \\
	Batch size & $16$ & $256$ & $16$ & $64$ & $64$ & $256$ \\
	Weight decay & $0.0$ & $0.0$ & $0.0$ & $0.0$ & $0.0005$ & $0.0005$ \\
	Momentum & $0.90$ & $0.0$ & $0.0$ & $0.0$ & $0.0$ & $0.0$ \\ \bottomrule
\end{tabular}%
}

	\vspace{5px}
	\caption{SGD hyperparameters of SGD-trained \texttt{FCN}s trained on \nlmsseven data that result in the best validation accuracy out of all $180$ hyperparameter combinations listed in~\Cref{subsec:geninfo}.}
	\label{table:hparams}
\end{table}

\clearpage
\section{Can we mitigate Simplicity Bias?}
\label{sec:fixes}
In this section, we investigate whether standard approaches for improving generalization error and adversarial robustness---ensembles and adversarial training---help in mitigating SB.

\subsection{Ensemble Methods}
We now study the extent to which ensembles mitigate SB and its adverse effect on generalization.
Specifically, we evaluate the performance of ensembles of fully-connected networks (\texttt{FCN}s) that are  trained on two datasets: \nlmsseven and \msfive.
Recall that the \nlmsseven data comprises one simple-but-noisy linear coordinate and multiple relatively complex $7$-slab coordinates that have no noise, whereas \msfive data comprises multiple noiseless $5$-slab coordinates only.
\vspace{-15px}
\begin{figure}[h]
	\subfloat[
	]{\includegraphics[width=0.29\linewidth]{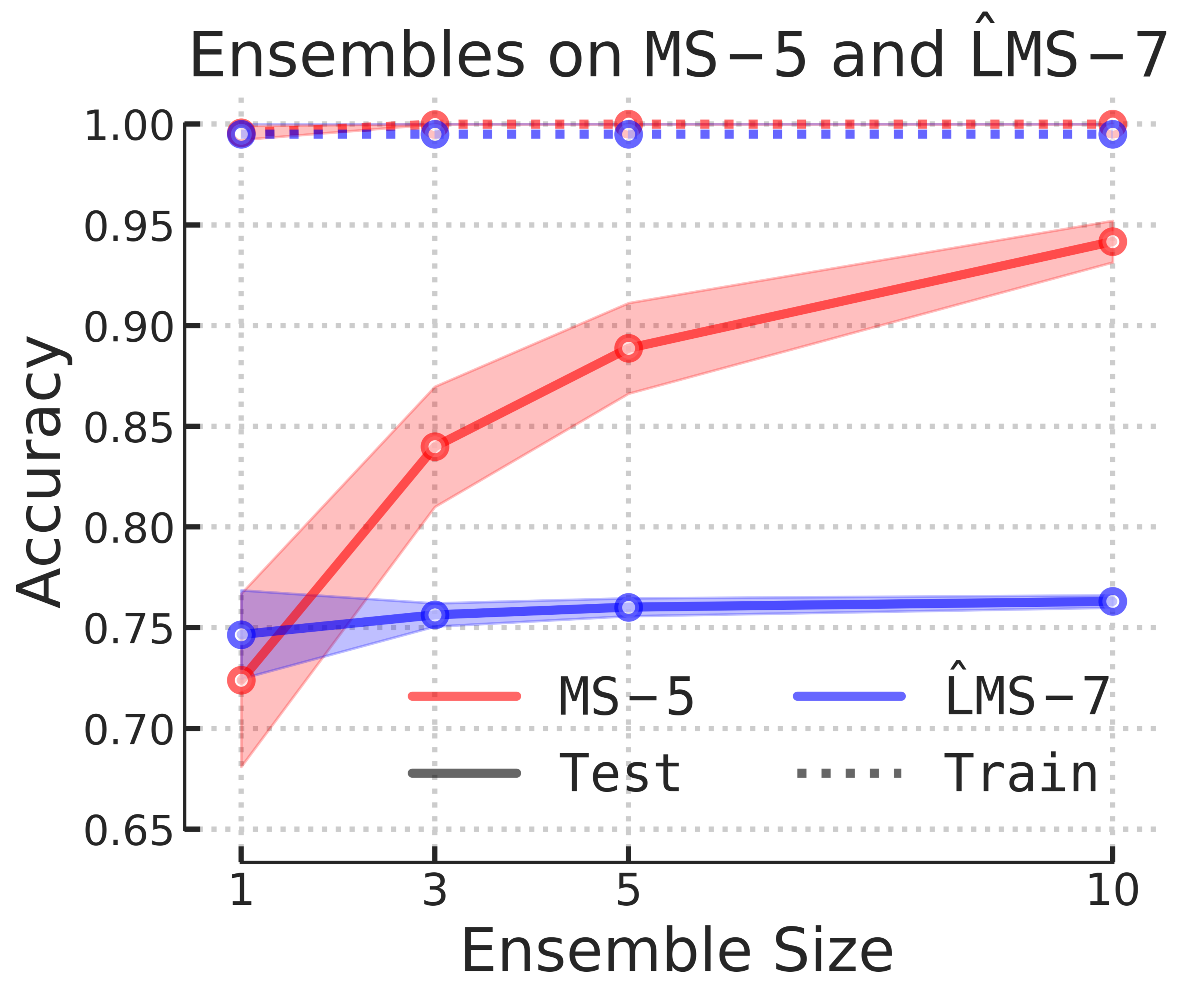}}
	\subfloat[
	]{\includegraphics[width=0.71\linewidth]{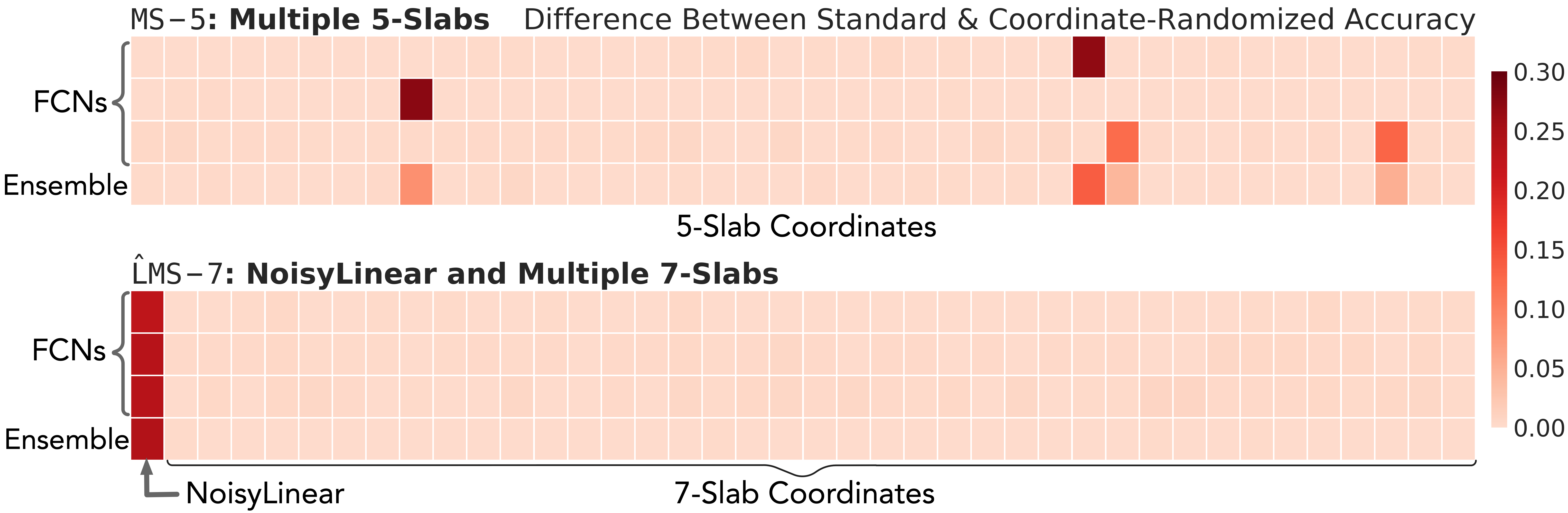}}
	\vspace{-10px}
	\caption{Ensembles improve performance on \msfive data that comprises features with equal predictive power and simplicity. However, as shown in (a), ensembles do not improve performance on \nlmsseven data that has a simple-but-noisy linear coordinate that has less predictive power than than the $7$-slab coordinates; this is because individual \texttt{FCN}s trained on \nlmsseven data exclusively rely on the noisy linear coordinate and consequently misclassify the same set of instances, as shown in subplot (b).}
	\label{fig:fixes}
\end{figure}

\vspace{-15px}
To better highlight the effect of ensembles on generalization, we choose a sample size (for both datasets) such that individual models (a) overfit to training data (i.e., non-zero generalization gap) but (b) still attain non-trivial test accuracy. We now discuss the performance of ensembles of independently trained models on \msfive and \nlmsseven datasets:
\begin{itemize}[leftmargin=*]
\item \textbf{\msfive data}: Recall that \msfive data comprises multiple independent 5-slab blocks, one in each coordinate, that have equal simplicity and predictive power. Thus, since all features have equal simplicity, independent SGD-trained \texttt{(100,2)-FCN} end up relying on different $5$-slab coordinates due to random initialization, as shown in~\Cref{fig:fixes}(b). As the training sample size is small, \texttt{FCN}s overfit to the training data and  attain approximately $75\%$ test accuracy, as shown in~\Cref{fig:fixes}. Consequently, as shown in~\Cref{fig:fixes}, ensembles of these models rely on all $5$-slab coordinates learned by the individual models and attain better test accuracy by aggregating model predictions and averaging out overfitting. For example,~\Cref{fig:fixes} shows that ensembles of size $5$ and $10$ improves generalization by approximately $15\%$ and $20\%$ respectively.
\item \textbf{\nlmsseven data}:	 Recall that \nlmsfive data comprises one simple-but-noisy linear block (with $50\%$ noise) and multiple independent $7$-slab blocks that have no noise. Now, due to extreme SB, every independently trained \texttt{FCN} exclusively latches on (and overfits to) the simpler-but-noisy linear block, as shown in~\Cref{fig:fixes}(b). As a result, all models collectively lack diversity and essentially learn the same decision boundary because of extreme SB. Therefore, ensembles of these models do not improve generalization because the independent models make misclassifications on the same instances. As shown in~\Cref{fig:fixes}(a), ensembles of size $3$, $5$ and $10$ do not improve generalization---the test accuracy remains $75\%$. 
\end{itemize}

\vspace{-6px}
The ensemble performance on \msfive data indicates that when datasets have \emph{multiple} equally simple features, ensembles of independently trained models mitigate SB to some extent by aggregating predictions of models that rely on simple features. 
Conversely, the ensemble performance on \nlmsseven data suggests that when datasets comprise few features that the more noisy and less predictive than the rest, ensembles may not improve generalization.
Our results also suggest that the generalization improvements using ensemble methods in practice may stem from combining multiple simple-but-noisy features (such as color, texture) and not by learning complex features (such as shape). 

\subsection{Adversarial Training}
\label{sec:advtrain}
We now investigate the extent to which adversarial training~\cite{madry2017towards} mitigates SB and its adverse effect on adversarial robustness using two datasets: \mnistcifar and \advms.

\subsubsection{Adversarially training \texttt{FCN}s  on \advms data}
 Now, we first introduce \advms, a variant of the \msfiveseven dataset, and then investigate if adversarially trained \texttt{FCN}s that improve adversarial robustness by some extent also mitigate extreme simplicity bias (SB).

\underline{\advms dataset}: Recall that $d$-dimensional \msfiveseven data, introduced in~\Cref{sec:prelims}, consists of $d-1$ $7$-slab coordinates and a single relatively simpler $5$-slab coordinate, all of which have perfect predictive power.
Similar to \msfiveseven data, the $d$-dimensional \advms data comprises $5$-slab and $7$-slab coordinates. 
Specifically, the first $\sfrac{d}{2}$ coordinates correspond to independent $5$-slabs, each with effective margin $\gamma_{5}$ and the other $\sfrac{d}{2}$ coordinates correspond to independent $7$-slabs with effective margin $\gamma_{7}$. 
In contrast to \msfiveseven data, the \advms dataset (a) comprises $\sfrac{d}{2}$ $5$-slab coordinates and (b) the $5$-slabs and $7$-slabs do not necessarily share the same effective margin. In our experiments below, we set $d=20$, $\gamma_5=0.05$ and $\gamma_7=0.15$. That is, we conduct our experiments on $20$-dimensional data in which the the simple features \SF and complex features \SFc correspond to the $10$ small-margin $5$-slab and $10$ large-margin $7$-slab coordinates respectively. 

\underline{Experiment setup}: 
In~\Cref{sec:extreme}, we observed that SGD-trained \texttt{FCN}s, due to extreme SB, exclusively rely on the simplest feature \SF and consequently learn small-margin classifiers that are highly vulnerable to adversarial attacks.
To mitigate SB and attain optimal adversarial robustness, fully-connected networks with width $w$ and depth $d$, $(w,d)$-\texttt{FCN}s, must learn maximum-margin classifiers and consequently rely on \emph{all} simple and complex features (i.e., features in \SF and \SFc).  
The margin of the $20$-dimensional \advms dataset described above is approximately $\gamma_{\text{data}}=0.62$, which implies that the maximum-margin classifier should exhibit robustness to $\ell_2$ adversarial perturbations that have norm $\epsilon < \gamma_{\text{data}}$.
Therefore, to check if adversarial training mitigates extreme SB, we evaluate the robustness of adversarially trained \texttt{FCN}s against $\ell_2$ adversarial perturbations that have norm $\epsilon < \gamma_{\text{data}}$. 
In addition to $\gamma_{\text{data}}$, let $\gamma_{\mathtt{S}}=0.30$ denote the maximum margin of the classifier that exclusively relies on all $5$-slab features in \SF.

Also note that (a) adversarial perturbations are generated using PGD attacks~\cite{madry2017towards}, (b) $(200,2)$-\texttt{FCN}s and $(1000,2)$-\texttt{FCN}s are expressive enough to learn the maximum-margin classifier on the $20$-dimensional \advms data, (c) \texttt{FCN}s are adversarially trained for $4000$ epochs with initial learning rate $0.1$ that decays by a multiplicative factor of $0.1$ after every $1000$ epochs, and (d) the training data comprises $6000$ data points, which is enough for SGD-trained $(1000,2)$-\texttt{FCN}s to learn $7$-slab coordinates and attain $100\%$ generalization.

\begin{figure}[b]
	\centering
	\includegraphics[width=\linewidth]{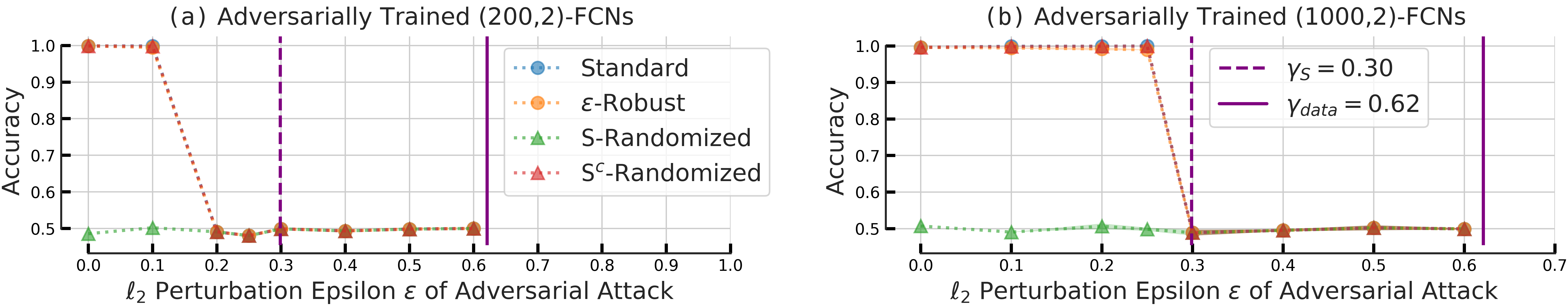}
	\vspace{-18px}
	\caption{ 
		Adversarially trained \texttt{FCN}s on \advms data exhibit adversarial robustness to some extent, but are (a) unable to mitigate extreme simplicity bias, as shown by the \{\SF,\SFc\}-randomized accuracies and (b) do not learn maximum-margin classifiers that attain optimal adversarial robustness (i.e., $100\%$ $\gamma_{\text{data}}$-robust accuracy). See~\Cref{sec:advtrain} for more detail.}
	\label{fig:adv_noise}
\end{figure}

\underline{Experiment results}:
Recall that the feature sets \SF and \SFc correspond to the set of ten $5$-slab and $7$-slab coordinates in the \advms dataset respectively. 
As shown in~\Cref{fig:adv_noise}, we evaluate the standard (blue), $\epsilon$-robust (orange), \SF-randomized (green) and \SFc-randomized (red) accuracies, defined in~\Cref{sec:prelims}, of \texttt{FCN}s with width $\{200,1000\}$ and depth $2$ that are adversarially trained with $\ell_2$ perturbation norm $\epsilon \leq \gamma_{\text{data}}$. The dashed and solid purple vertical bars denote $\gamma_{\mathtt{S}}$ and $\gamma_{\text{data}}$ respectively.
We make two key observations:
\begin{itemize}[leftmargin=*]
	\item \underline{Adversarial trained \texttt{FCN}s do not learn maximum-margin classifiers}. When $\epsilon \leq 0.1$, $(200,2)$-\texttt{FCN}s learn classifiers that attain $100\%$ standard and $\epsilon$-robust accuracies. However, when $\epsilon \geq 0.2$, due to optimization-related issues, adversarially trained $(200,2)$-\texttt{FCN}s are unable to learn a non-trivial classifier that obtains more than $50\%$ standard and $\epsilon$-robust accuracy. Increasing the model width from $200$ to $1000$ improves adversarial robustness to some extent---adversarially trained $(1000,2)$-FCNs learn classifiers with $100\%$ standard and robust accuracies when $\epsilon \leq 0.25$. However, when $\epsilon \geq \gamma_{\mathtt{S}}=0.3$ (dashed purple line), adversarially trained $(2000,2)$-\texttt{FCN}s are unable to learn $\epsilon$-robust classifiers as well. We note that further increasing the model width to $2000$ does not improve robustness. Consequently, adversarial training does not result in maximum-margin classifiers that have optimal adversarial robustness (i.e., classifiers with $100\%$ $\gamma_{\text{data}}$-robust accuracy) on \advms data. These results reconcile two phenomena observed in practice: larger capacity models can improve adversarial robustness~\cite{xie2019intriguing}, but large-epsilon adversarial training can ``fail" and result in trivial classifiers due to optimization-related issues~\cite{sitawarin2020improving}.
	\item \underline{Adversarial training does not mitigate extreme SB}. The \{\SF,\SFc\}-randomized accuracies of adversarially trained \texttt{FCN}s collectively show that adversarial training does not mitigate extreme SB. When the perturbation budget $\epsilon \leq \gamma_{\mathtt{S}}$, adversarially trained $(2000,2)$-\texttt{FCN}s exhibit robustness by \emph{exclusively} relying on \emph{multiple} ``simple" $5$-slab coordinates. That is, randomizing \SF drops the model accuracy to $50\%$, but randomizing the more complex $7$-slab coordinates has no effect on model accuracy. Conversely, when $\epsilon \geq \gamma_{\mathtt{S}}$, classifiers with $100\%$ $\epsilon$-robust accuracy must rely on features in \SF \emph{and} \SFc. However, as shown in~\Cref{fig:adv_noise}, when $\epsilon \geq \gamma_{\mathtt{S}}$, adversarial training fails and results in trivial classifiers that attain $50\%$ standard and robust accuracy.
\end{itemize}

To summarize, the experiment above shows that adversarially trained \texttt{FCN}s do not mitigate simplicity bias or learn maximum-margin classifiers that are optimally robust to $\ell_2$ adversarial attacks.

\subsubsection{Adversarially training \texttt{CNN}s on \mnistcifar data}
Recall that the \mnistcifar images, described in~\Cref{sec:prelims}, are vertically concatenations of ``simple" \mnist images and the more complex \cifar images. Next, we show that models that are $\ell_{\infty}$-adversarially trained on the \mnistcifar data improve $\epsilon$-robust accuracy (defined in~\Cref{sec:prelims}) but do not achieve the best possible $\epsilon$-robust accuracy due to extreme SB.

\Cref{table:advtrain_cifar} evaluates the standard, $\epsilon$-robust and \cifar-randomized accuracies of SGD-trained and adversarially trained MobileNetV2, ResNet50 and DenseNet121 using the \mnistcifar dataset. 
First, we observe that adversarial training with perturbation norm $0.3$ significantly improves $\epsilon$-robust accuracies over those of SGD-trained models, without degrading the models' standard test accuracies. However, the \cifar-randomized accuracies indicate that the adversarial training does not lead to reliance on the \cifar block---adversarially trained \texttt{CNN}s continue to remain invariant to the \cifar block even though it is almost fully predictive of the label. Consequently, these results suggest that adversarial training improves robustness but does not achieve the best $\epsilon$-robust accuracy on the \mnistcifar dataset.

To summarize, our experiments on \advms and \mnistcifar datasets show that while adversarial training does improve the $\epsilon$-robust accuracy over that of SGD-trained model, adversarially trained models continue to remain susceptible to extreme SB and consequently do not achieve maximum possible adversarial robustness.
\begin{table*}[h]
	\def\arraystretch{1.1}
\centering
\resizebox{\textwidth}{!}{
\begin{tabular}{cccccccccc}
\toprule
\multirow{2}{*}{Model} & \multirow{2}{*}{{$\ell_{\infty}$ budget $\varepsilon$ }} & \multicolumn{2}{c}{Test Accuracy} & & \multicolumn{2}{c}{$\epsilon$-Robust Accuracy} & & \multicolumn{2}{c}{\cifar-Randomized Accuracy} \\  \cmidrule{3-4} \cmidrule{6-7} \cmidrule{9-10}
	& & Standard SGD & $\ell_{\infty}$ Adv.  Training & &  Standard SGD & $\ell_{\infty}$ Adv.  Training & &  Standard SGD & $\ell_{\infty}$ Adv.  Training \\ \midrule
	\multirow{1}{*}{\texttt{MobileNetV2}} & $0.30$ & $0.999 \pm 0.001$ & $0.999 \pm 0.000$ & & $0.000 \pm 0.000$ & $0.991 \pm 0.000$ & & $0.493 \pm 0.005$ & $0.493 \pm 0.001$ \\
	\multirow{1}{*}{\texttt{DenseNet121}} & $0.30$ & $1.000 \pm 0.000$ & $0.999 \pm 0.000$ & & $0.000 \pm 0.000$ & $0.981 \pm 0.003$ & & $0.494 \pm 0.005$ & $0.501 \pm 0.003$ \\
	\multirow{1}{*}{\texttt{ResNet50}} & $0.30$ & $1.000 \pm 0.000$ & $0.999 \pm 0.001$ & & $0.001 \pm 0.000$ & $0.982 \pm 0.002$ & & $0.501 \pm 0.001$ & $0.499 \pm 0.002$ \\
	\bottomrule
	\end{tabular}
}

	\caption{Adversarial training on \mnistcifar: The table above presents standard, $\epsilon$-robust and \cifar-randomized accuracies of SGD-trained and adversarially trained MobileNetV2, DenseNet121 and ResNet50 models. While adversarial training significantly improves $\varepsilon$-robust accuracy, it does not encourage models to learn complex features (\cifar block in this case). The \cifar-randomized accuracies indicate that adversarially trained models do not mitigate extreme SB, as they exclusively rely on the \mnist block.}
	\label{table:advtrain_cifar}
\end{table*}

\clearpage
\section{Proof of Theorem 1}
\label{sec:proof}
In this section, we first re-introduce the data distribution and theorem. Then,
we describe the proof sketch and notation, before moving on to
the proof.

\textbf{\textbf{L}inear-\textbf{S}lab-\textbf{N}oise (\lsn) data}: The \lsn dataset is a stylized version of \lmsk that is amenable to theoretical analysis. In \lsn, conditioned on the label $y$, the first and second coordinates of $x$ are \emph{singleton} linear and $3$-slab blocks: linear and $3$-slab blocks have support on $\{\minus 1,1\}$ and $\{\minus 1,0,1\}$ respectively. The remaining coordinates are standard gaussians and not predictive of the label. Each data point $(x_i,y_i) \in \Re^d \times \{-1,1\}$ from \lsn can be sampled as follows:
\begin{align*}
y_i &= \pm 1,\ \mbox{ w.p. } 1/2, \ \ \ \varepsilon_i = \pm 1,\  \mbox{ w.p. } 1/2,\\
x_{i1} &= y_i & \text{(Linear coordinate)},\\
x_{i2} &= \Big(\frac{y_i+1}{2}\Big) \varepsilon_i & \text{(Slab coordinate)},\\
x_{i,3:d} &\sim \mathcal{N}(0, I_{d-2}) & \text{($d \minus 2$ Noise coordinates)}.
\end{align*}
According to~\Cref{thm:main} (re-stated), one-hidden-layer ReLU neural networks trained with standard mini-batch gradient descent (GD) on the \lsn dataset provably learns a classifier that \emph{exclusively} relies on the ``simple" linear coordinate, thus exhibiting simplicity bias at the cost of margin.
\begin{retheorem}
	Let $f(x) = \sum_{j=1}^{k} v_j \cdot \textrm{ReLU}(\sum_{i=1}^{d} w_{i,j} x_i)$ 	denote a one-hidden-layer neural network with $k$ hidden units and ReLU activations.
	Set $v_{j} = \pm \sfrac{1}{\sqrt{k}}$ w.p. $\sfrac{1}{2}$ $\forall j \in [k]$.
	Let $\{(x^{i},y^{i})\}^m_{i=1}$ denote i.i.d. samples from \lsn where $m \in [c d^2, d^\alpha / c]$ for some $\alpha > 2$.
	Then, given $d > \Omega(\sqrt{k}\log k)$ and initial $w_{ij} \sim \mathcal{N}(0,\frac{1}{dk \log^4 d})$, after $O(1)$ iterations, mini-batch gradient descent (over $w$) with hinge loss, {step size $\eta = \Omega{{(\log d)^{\sfrac{\minus 1}{2}}}}$}, mini-batch size $\Theta(m)$, satisfies:
	\vspace{-7.5px}
	\begin{itemize}[leftmargin=*]
		\itemsep 0pt
		\topsep 0pt
		\parskip 0pt
		\item Test error is at most $\sfrac{1}{\textrm{poly}(d)}$
		\item The learned weights of hidden units $w_{ij}$ satisfy:
		\small
		$$ \underbrace{\abs{w_{1j}} = {\frac{2}{\sqrt{k}} \left(1-\frac{c}{\sqrt{\log d}}\right) + \order{\frac{1}{\sqrt{dk} \log d}}}}_{{\textstyle \textbf{Linear Coordinate}}} ,\; \underbrace{\abs{w_{2,j}} = {\order{\frac{1}{\sqrt{dk} \log d}}}}_{{\textstyle \textbf{3-Slab Coordinate}}} ,\; \underbrace{\norm{w_{3:d, j}} = \order{\frac{1}{\sqrt{k} \log d}}}_{{\textstyle \textbf{$d\minus 2$ Noise Coordinates}}} $$
		\normalsize
	\end{itemize}
\vspace{-5px}
with probability greater than $1-\frac{1}{\textrm{poly}(d)}$. Note that $c$ is a universal constant.
\end{retheorem}

\textbf{Proof Sketch}
Since the number of iterations $t = O(1)$, we partition the dataset into $t$ minibatches each of size $n := m/t$ samples. This means that each iteration uses a fresh batch of $n$ samples and the $t$ iterations together form a single pass over the data.
The overall outline of the proof is as follows. If the step size is $\eta$, then for $t \lesssim \frac{4}{\eta}$ iterations, with probability $\geq 1- \frac{1}{\textrm{poly}(d)}$,
\begin{itemize}[leftmargin=*]
	\item Lemma~\ref{lemma:active} shows that the hinge loss is ``active" (i.e., $yf(x) < 1$) for all data points in a given batch.
	\item Under this condition, we derive closed-form expressions for \emph{population} gradients in Lemmas~\ref{lemma:grad_linear},~\ref{lemma:grad_slab} and~\ref{lemma:grad_noise}.
	\item Lemma~\ref{lemma:main_ih} uses the above lemmas to establish precise estimates of the linear, slab and noise coordinates for all iterations until $t$.
\end{itemize}
The proof is organized as follows.
Appendix~\ref{sec:proof_induction} presents the main lemmas that will directly lead to Theorem~\ref{thm:main}. Appendix~\ref{sec:proof_gradients} derives closed form expressions for population gradients and Appendix~\ref{sec:proof_misc} presents auxiliary lemmas that are useful in the main proofs.

\textbf{Notation}
Recall that $f(x) = \sum^k_{j=1} v_j \cdot \relu(\sum^d_{i=1} w_{ij} x_i) = v^T \relu(W^Tx)$ where $W \in \mathds{R}^{d \times k}$ and $v \in \R^k$. Note that $w_i = [w_{1i} \, w_{2i}, \cdots w_{di}]^T$ is the $i^{th}$ column in $W$. Let $\bar{w}_i$ and $\bar{x}_j$ denote the $w_{3:d,i}$ and $x_{3:d,j}$ respectively. Also, let $\S_n = \{(x_i, y_i)\}^n_{i=1}$ denote a set of $n$ i.i.d. points randomly sampled from $\texttt{LSN}$. For simplicity, we also assume $|\{i : v_i = \sfrac{1}{\sqrt{k}} \}| = |\{i : v_i = -\sfrac{1}{\sqrt{k}} \}| = \sfrac{k}{2}$. We can now define the loss function as $\loss_f(\S_n) = \sfrac{1}{n} \sum^n_i \ell(x_i, y_i)$, where $\ell(x,y)=\max(0,1-yf(x))$ denotes the hinge loss. For notational simplicity, we use $X = \mu \pm \delta$ and $|X-\mu| \leq \delta$ interchangeably. Also let $\varphi$ and $\phi$ denote the probability density function and cumulative distribution function of standard normal distribution.

\vspace{-12px}
\begin{proof}[Proof of Theorem~\ref{thm:main}]
The proof directly follows from Lemma \ref{lemma:main_ih} and Lemma \ref{lemma:active}. In Lemma \ref{lemma:main_ih}, we show that the weights in the linear coordinate are $\Omega(\sqrt{d})$ larger than the weights in the slab and noise coordinates. Applying Lemma \ref{lemma:main_ih} at $\hat{t}=\floor{\frac{4}{\eta}(1-\frac{c_n}{\sqrt{\log d}})}$ gives the following result:
\begin{align*}
w^{(\hat{t})}_{1i} \overset{(a)}{=} \frac{2}{\sqrt{k}}(1-\frac{c_n}{\sqrt{\log d}}) + \bigo(\frac{1}{\sqrt{dk}\log d}) \text{ and }
|w^{(\hat{t})}_{2i}| \overset{(a)}{=} \bigo(\frac{1}{\sqrt{dk}\log d}) \text{ and }
||\bar{w}^{(\hat{t})}_{i}| &\overset{(a)}{=} \bigo(\frac{1}{\sqrt{k}\log d})
\end{align*}
where $(a)$ is due to $c_0(1+\hat{c})^{t} \leq c_0 e^{\hat{c}t} \leq c_0 e^1 = \bigo(1)$.

The $0-1$ error of the function $f$ at timestep $\hat{t}$ is small as well, because we can directly use Lemma \ref{lemma:active} to get $\Pr(yf(x) < 0) = \sfrac{2}{c^3d^6}$. Therefore, the $0-1$ error is at most $\frac{2}{c^3d^6} = \bigo(\frac{1}{d^6})$.
\end{proof}

\subsection{Proof by Induction}
\label{sec:proof_induction}
In this section, we use proof by induction to show that for the first $t=O(\sfrac{1}{\eta})$ steps, (1) the hinge loss is ``active" for all data points (Lemma~\ref{lemma:active}) and (2) hidden layer weights in the linear coordinate are $\Omega(\sqrt{d})$ larger than the hidden layer weights in the slab and noise coordinates (Lemma~\ref{lemma:main_ih}). \\

\begin{lemma}
\label{lemma:main_ih}
Let $|\S_n| \in [cd^{2}, d^{\alpha}/c]$ and initialization $w_{ij} \sim \mathcal{N}(0, \sfrac{1}{dk \log^2 d})$. Also let $\hat{c} = \sfrac{\eta}{4}$, $c_0 = 2$ and $c_n=5\sqrt{\alpha}c_0(1+\hat{c})^t$. Then, for all $t \leq \frac{4}{\eta}(1-\sfrac{c_n}{\sqrt{\log d}})$, $d \geq \exp((\sfrac{8c_n}{\eta})^2)$, $\sqrt{\sfrac{d}{\log^3(d)}} > \sfrac{24\sqrt{k}}{c_0 c}$ and $i \in [k]$, w.p. greater than $1-O(\frac{1}{d^2})$, we have:
\begin{align}
& y_i f(x_i) \leq 1 \, \forall (x_i, y_i) \in \S_n \label{eq:active_ih} \\
& w^{(t)}_{1i} = \frac{t\eta v_i}{2} \pm \frac{c_0 (1+\hat{c})^t}{\sqrt{dk}\log d} \label{eq:lin_ih} \\
&|w^{(t)}_{2i}| \leq \frac{c_0 (1+\hat{c})^t}{\sqrt{dk}\log d} \label{eq:slab_ih} \\
&||\bar{w}^{(t)}||_2 \leq \frac{c_0 (1+\hat{c})^t}{\sqrt{k} \log d}  \label{eq:noise_ih}
\end{align}
\end{lemma}

\begin{proof}
First, we prove that equations \eqref{eq:lin_ih}, \eqref{eq:slab_ih} \& \eqref{eq:noise_ih} hold at initialization (i.e., $t=0$) with high probability. Using \ref{lemma:norm_max} and \ref{lemma:norm_l2}:
$$ \max_{i \in \{1,2\}} \max_{j \leq k} |w_{ij}| \leq \frac{2}{\sqrt{dk}\log d} \;  \text{ and } \; \max_{i \leq k} ||\bar{w}_i|| \leq \frac{2}{\sqrt{k}\log d} \quad \text{w.p. } 1-\frac{2}{d^4}$$
Therefore, $w_{1i} = \frac{(0)\eta v_i}{2} \pm \frac{c_0 (1+\hat{c})^0}{\sqrt{dk}\log d}$ and $|w_{2i}| \leq \frac{c_0 (1+\hat{c})^0}{\sqrt{dk}\log d}$ and $||\bar{w}_i|| \leq \frac{c_0(1+\hat{c})^0}{\sqrt{k}\log d}$. Since equations \eqref{eq:lin_ih}, \eqref{eq:slab_ih} \& \eqref{eq:noise_ih} hold at $t=0$, we can use Lemma \ref{lemma:active} to show that the hinge loss is ``active" with high probability:$$y_if(x_i) = \pm \frac{c_n}{\sqrt{\log d}} < 1 \quad \text{when } d \geq \exp(c^2_n)$$

Now, we assume that the inductive hypothesis---equations \eqref{eq:active_ih}, \eqref{eq:lin_ih}, \eqref{eq:slab_ih} and \eqref{eq:noise_ih}---is true after every timestep $\tau$ where $\tau \in \{0, \cdots, t\}$.

We now prove that the inductive hypothesis is true at timestep $t+1$, after applying gradient descent using the $(t+1)^{th}$ batch. Since z\eqref{eq:active_ih} holds at timestep $t$, we can use the closed-form expression of the gradient along the linear coordinate (lemma \ref{lemma:grad_linear}) to prove that equation \eqref{eq:lin_ih} holds at timestep $t+1$ as well:
\small
\begin{align*}
w^{(t+1)}_{1i} &= w^{(t)}_{1i} + \frac{\eta v_i}{4}\Big[2+\phi\Big(\frac{w_{1i}+w_{2i}}{||\bar{w}_i||}\Big) + \phi\Big(\frac{w_{1i} \minus w_{2i}}{||\bar{w}_i||}\Big) - 2\phi\Big(\frac{ w_{1i}}{||\bar{w}_i||}\Big) \Big] \pm \frac{5\eta v_i}{d} \sqrt{\frac{\log (cd^2)}{c}} \\
&= w^{(t)}_{1i} + \frac{\eta v_i}{2} + \frac{\eta v_i}{2} |w^{(t)}_{2i}| \cdot \max_{|\delta| \leq |w^{(t)}_{2i}|} \frac{1}{||\bar{w}^t||} \varphi \Big(\frac{w^{(t)}_{1i}+\delta}{||\bar{w}^{(t)}_i||}\Big) \pm \frac{5\eta v_i}{d} \sqrt{\frac{\log (cd^2)}{c}} \\
&\overset{(a)}{=} \frac{t\eta v_i}{2} \pm \frac{c_0 (1+\hat{c})^t}{\sqrt{dk}\log d} + \frac{\eta v_i}{2} + \frac{\eta v_i}{2} \frac{c_0 (1+\hat{c})^t}{\sqrt{dk}\log d} \pm \frac{5\eta v_i}{d} \sqrt{\frac{\log (cd^2)}{c}} \\
&\overset{(b)}{=} \frac{(t+1)\eta v_i}{2} \pm \frac{c_0 (1+\hat{c})^t}{\sqrt{dk}\log d} \pm \eta v_i \frac{c_0 (1+\hat{c})^t}{\sqrt{dk}\log d} = \frac{(t+1)\eta v_i}{2} \pm \frac{c_0 (1+\hat{c})^t(1+\eta v_i)}{\sqrt{dk}\log d} \\
&\overset{(c)}{=} \frac{(t+1)\eta v_i}{2} \pm \frac{c_0 (1+\hat{c})^{t+1}}{\sqrt{dk}\log d}
\end{align*} \normalsize
where $(a)$ is via equation \eqref{eq:cdf2} in Lemma \ref{lemma:density}, $(b)$ is because $\sfrac{d}{\log^3(d)} \geq \sfrac{20}{c_0 e^{1} \sqrt{c}}$ and $(c)$ is due to $\eta v_i \leq \hat{c}$.

Similarly, since equation \eqref{eq:active_ih} holds at timestep $t$ (via the inductive hypothesis), we can use the closed-form expression of the gradient along the slab coordinate (lemma \ref{lemma:grad_slab}) to show that the weights in the slab (i.e., second) coordinate are small (equation \eqref{eq:slab_ih}) at timestep $t+1$ as well:
\small
\begin{align*}
w^{(t+1)}_{2i} &= w^{(t)}_{2i} + \frac{\eta v_i}{4}\Big[\phi\Big(\frac{w_{1i}+w_{2i}}{||\bar{w}_i||}\Big) - \phi\Big(\frac{w_{1i} \minus w_{2i}}{||\bar{w}_i||}\Big)\Big] \pm \frac{5\eta v_i}{d} \sqrt{\frac{\log (cd^2)}{c}} \\
&= w^{(t)}_{2i} + \frac{\eta v_i}{2} |w^{(t)}_{2i}| \cdot \max_{|\delta| \leq |w^{(t)}_{2i}|} \frac{1}{||\bar{w}^t||} \varphi \Big(\frac{w^{(t)}_{1i}+\delta}{||\bar{w}^{(t)}_i||}\Big) \pm \frac{5\eta v_i}{d} \sqrt{\frac{\log (cd^2)}{c}} \\
&\overset{(a)}{=} \pm \frac{c_0 (1+\hat{c})^t}{\sqrt{dk}\log d} \pm \frac{\eta v_i}{2} \frac{c_0 (1+\hat{c})^t}{\sqrt{dk}\log d} \pm  \frac{\eta v_i}{2} \frac{c_0 (1+\hat{c})^t}{\sqrt{dk}\log d} \leq \pm  \frac{c_0 (1+\hat{c})^{t+1}}{\sqrt{dk}\log d}
\end{align*} \normalsize
where $(a)$ is due to equations \eqref{eq:cdf1} in Lemma \ref{lemma:density}, \eqref{eq:slab_ih} and $\sfrac{d}{\log^3(d)} \geq \sfrac{20}{c_0 e^{1} \sqrt{c}}$.

Finally, we can use the closed-form expression of the gradient along the noise  coordinate (lemma \ref{lemma:grad_noise}) to prove that the norm of the gradient along the noise coordinates (i.e., coordinates $3$ to $d$) is small (equation \eqref{eq:noise_ih}) at timestep $t+1$:
\footnotesize
\begin{align}
\bar{w}^{(t+1)}_i = \bar{w}^{(t+1)}_i + \underbrace{\frac{\eta v_i}{4}\Big[\varphi\Big(\frac{w_{1i}+w_{2i}}{||\bar{w}_i||}\Big) + \varphi\Big(\frac{w_{1i}\minus w_{2i}}{||\bar{w}_i||}\Big) - 2\varphi\Big(\frac{w_{1i}}{||\bar{w}_i||}\Big) \Big]}_{\bar{\mathcal{G}}_1} \frac{\bar{w}^{(t)}_i}{ ||\bar{w}_i||} \pm \underbrace{\frac{3\eta |v_i|\log (\sqrt{c}d)}{\sqrt{c}d} \frac{\bar{w}_i}{||\bar{w}_i||} \pm \frac{6\eta |v_i|}{\sqrt{cd}} u^{\perp}_i}_{\bar{\mathcal{G}}_2} \nonumber
\end{align} \normalsize
We first show that the the first part of the noise gradient, $\bar{\mathcal{G}}_1$, is at most $\sfrac{\eta v_i}{2}$:
\small
\begin{align*}
&\frac{\eta v_i}{4}\Big[\varphi\Big(\frac{w_{1i}+w_{2i}}{||\bar{w}_i||}\Big) + \varphi\Big(\frac{w_{1i}\minus w_{2i}}{||\bar{w}_i||}\Big) - 2\varphi\Big(\frac{w_{1i}}{||\bar{w}_i||}\Big) \Big]\\
&= \frac{\eta v_i}{4} \underbrace{\frac{1}{||\bar{w}^{(t)}_i||} \varphi \Big(\frac{w^{(t)}_{1i}}{||\bar{w}^{(t)}_i||} \Big)}_{\leq 1 (\text{see eq. \ref{eq:pdf1} in Lemma \ref{lemma:density}})} \underbrace{\Big[2 - \varphi \Big(\frac{w^{(t)}_{2i}}{||\bar{w}^{(t)}_i||} \Big)\Big(\exp(\frac{w^{(t)}_{1i}w^{(t)}_{2i}}{||\bar{w}^{(t)}_i||^2}) + \exp(\frac{-w^{(t)}_{1i}w^{(t)}_{2i}}{||\bar{w}^{(t)}_i||^2})\Big)\Big]}_{\leq 2}
\leq \frac{\eta v_i}{2}
\end{align*} \normalsize
Next, we show that the $\ell_2$ norm of the second part of the noise gradient, $||\bar{\mathcal{G}}_2||$, is $O(\sfrac{1}{\sqrt{d}})$:
$$||B|| \leq 3\eta |v_i| \frac{\log(\sqrt{c}d)}{\sqrt{c}d} + \frac{6\eta|v_i|}{\sqrt{cd}} \leq \frac{12\eta |v_i|}{\sqrt{cd}}$$
Now, we can use the upper bounds on $\mathcal{G}_1$ and $\mathcal{G}_2$ to show that the $\ell_2$ norm of the gradient along the noise gradients is small as well:
\begin{align*}
||\bar{w}^{(t+1)}_i|| &\leq ||\bar{w}^{(t)}_i|| + \frac{\eta v_i}{2} ||\bar{w}^{(t)}_i|| + \frac{12\eta |v_i|}{\sqrt{cd}} \overset{(a)}\leq ||\bar{w}^{(t)}_i|| + \frac{\eta v_i}{2} ||\bar{w}^{(t)}_i|| + \frac{\eta v_i}{2} ||\bar{w}^{(t)}_i|| \\
&\overset{(b)}{\leq} \frac{c_0(1+\hat{c})^t(1+\eta v_i)}{\sqrt{k}\log d} \overset{(c)}{\leq} \frac{c_0(1+\hat{c})^{t+1}}{\sqrt{k}\log d}
\end{align*}
where $(a)$ is because $\sfrac{d}{\log d} \geq (\sfrac{24\sqrt{k}}{c_0c})^2$, $(b)$ is due to equation \eqref{eq:noise_ih} and $(c)$ is because $\eta v_i \leq \hat{c}$.
\end{proof}

Since equations \eqref{eq:lin_ih}, \eqref{eq:slab_ih} \& \eqref{eq:noise_ih} hold at timestep $t$ (from Lemma \ref{lemma:main_ih}), we can show that the hinge loss is positive (i.e., $yf(x)<1$) for all data points with high probability as well. \\

\begin{lemma}
\label{lemma:active}
Let $\S_n$ denote a set of $n \in [cd^{2}, d^{\alpha}/c]$ i.i.d. samples from $\texttt{LSN}$, where $\alpha > 2$ and $c>1$.
Suppose equations \eqref{eq:lin_ih}, \eqref{eq:slab_ih} \& \eqref{eq:noise_ih} hold at timestep $t$.
Also let $d \geq \exp((\frac{8c_n}{\eta})^2)$ where $c_n=5\sqrt{\alpha}c_0(1+\hat{c})^t$.
Then, w.p. greater than $1-\frac{2}{c^3d^6}$, we have:
\begin{align}
	 y_if(x_i) = \frac{t\eta}{4} \pm \frac{c_n}{\sqrt{\log d}} = (t \pm \sfrac{1}{2})\frac{\eta}{4}  \quad \forall (x_i, y_i) \in \S_n \label{eq:active}
\end{align}
\end{lemma}
\begin{proof}
We use equations \eqref{eq:lin_ih}, \eqref{eq:slab_ih} \& \eqref{eq:noise_ih} to obtain simplify the dot product between $w^{(t)}_i$ \& $x_j$ and the indicator $\ind{w^{(t)}_i \cdot x_j \geq 0}$. First, we show that the dot product between $w^{(t)}_i$ and $x_j$ is in the band $\frac{t\eta v_i y_j}{2} \pm  \frac{c_n}{\sqrt{k \log d}}$ with high probability:
\small
\begin{align}
w^{(t)}_i \cdot x_j &= w^{(t)}_{1i} y_j + w^{(t)}_{2i} \frac{y_j+1}{2} \varepsilon_j + \bar{w}^{(t)}_i \cdot \bar{x}_j \nonumber
\overset{(a)}{=} w^{(t)}_{1i} y_j + w^{(t)}_{2i} \frac{y_j+1}{2} \varepsilon_j + ||\bar{w}^{(t)}_i|| Z_j \\
&\overset{(b)}{=} w^{(t)}_{1i} y_j + w^{(t)}_{2i} \frac{y_j+1}{2} \varepsilon_j \pm ||\bar{w}^{(t)}_i|| \sqrt{8\alpha\log d} & \text{w.p. } 1-\frac{2}{d^6} \nonumber \\
&{=} \frac{t\eta v_i y_j}{2} \pm  \frac{c_0 (1+\hat{c})^t}{\sqrt{dk}\log d} (y_j + \frac{y_j+1}{2}\varepsilon_j) \pm \frac{c_0(1+\hat{c})^t}{\sqrt{k} \log d} \sqrt{8\alpha\log d} & \text{via eq. \ref{eq:lin_ih}, \ref{eq:slab_ih},  \ref{eq:noise_ih}} \nonumber \\
&\overset{(c)}{=} \frac{t\eta v_i y_j}{2} \pm  \frac{2c_0 (1+\hat{c})^t}{\sqrt{dk}\log d}  \pm \frac{3\sqrt{\alpha}c_0(1+\hat{c})^t}{\sqrt{k \log d}} = \frac{t\eta v_i y_j}{2} \pm  \frac{c_n}{\sqrt{k \log d}} \label{eq:wx_dot}  \\
&= (ty_j \pm \sfrac{1}{4})\frac{\eta v_i}{2} \quad \text{ when } d \geq \exp((\frac{8c_n}{\eta})^2) \label{eq:wx_dot_simple}
\end{align} \normalsize
where $(a)$ is because $\bar{w}^{(t)}_i \cdot \bar{x}_j = ||\bar{w}^{(t)}_i|| \mathcal{N}(0,1)$, $(b)$ is via lemma \ref{lemma:norm_max} \& $c>1$, and $(c)$ is because $(y_j + \frac{y_j+1}{2}\varepsilon_j) < 2$. Next, when $d \geq \exp((\frac{8c_n}{\eta})^2)$, we can simplify $\ind{w^{(t)}_i \cdot x_j \geq 0}$ as follows:
\small
\begin{align}
\ind{w^{(t)}_i \cdot x_j \geq 0} \overset{eq. \ref{eq:wx_dot_simple}}\leq \ind{
(ty_j \pm \sfrac{1}{4})\frac{\eta v_i}{2}  \geq 0} \leq \begin{cases}
      1, & \text{if}\ t=0 \\
      1, & \text{if}\ t>0 \text{ and } y_j v_i \geq 0 \\
      0, & \text{if}\ t>0 \text{ and } y_j v_i < 0
    \end{cases}
    \leq \ind{t=0 \vee yv_i \geq 0} \label{eq:wx_ind}
\end{align} \normalsize
We can now use equations \eqref{eq:wx_dot} \& \eqref{eq:wx_ind} to show that $y_j f^{(t)}(x_j)$ is in the band $\sfrac{t\eta}{4} \pm O(\sfrac{1}{\sqrt{\log d}})$ with high probability:
\small
\begin{align}
y_j f^{(t)}(x_j) &= \sum^k_{i=1} y_j v_i \cdot \relu(w_i \cdot x_j) = \sum^k_{i=1} v_i \ind{t=0 \vee yv_i \geq 0} \big(\frac{t\eta v_i}{2} \pm \frac{c_n}{\sqrt{k \log d}}\big) \nonumber \\
&= \sum^k_{i=1} \ind{t=0 \vee yv_i \geq 0} \big(\frac{t\eta}{2k} \pm \frac{c_n}{k \sqrt{\log d}}\big) \overset{(a)}{=} \begin{cases}
      \pm k \frac{c_n}{k \sqrt{\log d}}, & \text{if}\ t=0 \\
      \frac{k}{2} \big(\frac{t\eta}{2k} \pm \frac{c_n}{k \sqrt{\log d}}\big), & \text{if}\ t>0  \\
    \end{cases} \\
&= \frac{t\eta}{4} \pm \frac{c_n}{\sqrt{\log d}} \overset{(b)}= (t \pm \sfrac{1}{2})\frac{\eta}{4} \nonumber
\end{align} \normalsize
where $(a)$ is due to $|\{v_i \; | \; v_i > 0\}| = |\{v_i \; | \; v_i < 0\}| = \sfrac{k}{2}$ and $(b)$ follows from $\sfrac{c_n}{\sqrt{\log d}} \leq \sfrac{\eta}{8}$ when $d \geq \exp((\sfrac{8c_n}{\eta})^2)$
\end{proof}

\begin{lemma}
\label{lemma:density}
If equations \eqref{eq:lin_ih}, \eqref{eq:slab_ih} \& \eqref{eq:noise_ih} hold at timestep $t$, $d > \exp((\frac{4c_0e^{1}}{\eta})^2)$ and $\sfrac{d}{\log d} > \sqrt{k}$, we have:
\small
\begin{align}
	\max_{|\delta| \leq |w^{(t)}_{2i}|} \frac{1}{||\bar{w}^t||} \varphi \Big(\frac{w^{(t)}_{1i}+\delta}{||\bar{w}^{(t)}_i||}\Big) &\leq 1 \label{eq:density} \\
	\Big| \phi\Big(\frac{w^{(t)}_{1i}+w^{(t)}_{2i}}{||\bar{w}^{(t)}||}\Big) - \phi\Big(\frac{w^{(t)}_{1i}-w^{(t)}_{2i}}{||\bar{w}^{(t)}||}\Big)\Big| &\leq \frac{2c_0 (1+\hat{c})^t}{\sqrt{dk}\log d} \label{eq:cdf1} \\
	\Big| \phi\Big(\frac{w^{(t)}_{1i}+w^{(t)}_{2i}}{||\bar{w}^{(t)}||}\Big) - \phi\Big(\frac{w^{(t)}_{1i}}{||\bar{w}^{(t)}||}\Big)\Big| &\leq \frac{c_0 (1+\hat{c})^t}{\sqrt{dk}\log d} \label{eq:cdf2} \\
	\frac{1}{||\bar{w}^t||} \varphi \Big(\frac{w^{(t)}_{1i}}{||\bar{w}^{(t)}_i||}\Big) &\leq \frac{c_0 (1+\hat{c})^t}{\sqrt{dk}\log d} \label{eq:pdf1}
\end{align}\normalsize
\end{lemma}
\begin{proof}
Let $g_z(x) = \frac{1}{x} \varphi(\frac{z}{x})$ and $h(x) = \max_{|\delta| \leq |w^{(t)}_{2i}|} \frac{1}{x} \varphi(\frac{w^{(t)}_{1i}+\delta}{x}) = \max_{|\delta| \leq |w^{(t)}_{2i}|} g_{w^t_{1i}+\delta}(x)$. To prove \Cref{eq:density}, we show that an upper bound on $||w^{(t)}||$ is less than a lower bound on $\arg\max_x h(x)$, which subsequently implies that $h(||w^{(t)}||) < \max_x h(x)$ because $h$ is an increasing function for all $|x| \leq \arg\max_x h(x)$.

First, we find the maximizer $x^*$ of $h(x)$ as follows:
\begin{align*}
\max_x h(x) &= \max_x \max_{|\delta| \leq |w^{(t)}_{2i}|} g_{w^{(t)}_{1i}+\delta}(x) \overset{(a)}{=} \max_{|\delta| \leq |w^{(t)}_{2i}|} \frac{e^{-1}}{|w^{(t)}_{1i}+\delta|} & \text{when } x^*=|w^{(t)}_{1i}+\delta|
\end{align*}
where $(a)$ follows from lemma \ref{lemma:maxima}. Next, we lower bound the maximizer $x^*$ of $h(x)$:
\small
\begin{align*}
&x^*=|w^{(t)}_{1i}+\delta|\geq |w^{(t)}_{1i}+w^{(t)}_{2i}|\geq \Big||w^{(t)}_{1i}|-|w^{(t)}_{2i}|\Big|
\overset{(a)}{\geq} \Big||\frac{t\eta v_i}{2} \pm \frac{c_0 (1+\hat{c})^t}{\sqrt{dk}\log d}|-|\frac{c_0 (1+\hat{c})^t}{\sqrt{dk}\log d}|\Big| \\
&\overset{(b)}{\geq} \Big||\frac{t\eta v_i}{2} \pm \frac{\eta v_i}{8}|-|\frac{\eta v_i}{8}|\Big| \geq \Big|t-\sfrac{1}{2}\Big|\frac{\eta v_i}{2} \geq \frac{\eta v_i}{4}
\end{align*} \normalsize
where $(a)$ follows from the weights in the linear and slab coordinate at timestep $t$ (equations \eqref{eq:lin_ih} \& \eqref{eq:slab_ih}) and $(b)$ is because $\frac{c_0 (1+\hat{c})^t}{\sqrt{dk}\log d} \leq \frac{\eta v_i}{8}$ when $\sqrt{d} \geq \frac{8c_0 e^{1}}{\eta}$. Therefore, $\arg\max_x h(x) \geq \sfrac{\eta v_i}{4}$. We can use the upper bound on the $\ell_2$ norm of the gradient along the noise coordinates (equation \eqref{eq:noise_ih}) and $d \geq \exp(\frac{4c_0e^{1}}{\eta})$ to show that $||\bar{w}^{(t)}_i||$ is less than $x^*$:
\begin{align*}
||\bar{w}^{(t)}_i|| \leq \frac{c_0 (1+\hat{c})^t}{\sqrt{dk}\log d} \leq \frac{\eta v_i}{4} \leq \arg\max_x h(x)
\end{align*}
From lemma \ref{lemma:maxima}, we know that $h(x)$ is an increasing function for all $|x| < x^*$. This implies that $h(||w^{(t)}_{i}||) \leq h(\frac{c_0 (1+\hat{c})^t}{\sqrt{dk}\log d}) \leq h(\frac{\eta v_i}{4}) \leq h(x^*)$. Therefore, when $d \geq \exp((\frac{4c_0e^{1}}{\eta})^2)$ and $\sfrac{d}{\log d} \geq \sqrt{k}$, we obtain the desired result as follows:
\begin{align*}
\max_{|\delta| \leq |w^{(t)}_{2i}|} \frac{1}{||\bar{w}^t||} \varphi \Big(\frac{w^{(t)}_{1i}+\delta}{||\bar{w}^{(t)}_i||}\Big) = h(||w^{(t)}_{i}||) \leq h(\frac{c_0 e^{1}}{\sqrt{dk}\log d}) \leq \frac{\sqrt{k} \log d}{c_0 e^{1}} \frac{1}{d^{(\frac{\eta}{4c_0e^{1}})^2 \log d}} \leq 1
\end{align*}
Now, we can prove equations \eqref{eq:cdf1}, \eqref{eq:cdf2} and \eqref{eq:pdf1} using equation \eqref{eq:density} as follows:
\small
\begin{align}
&\Big| \phi\Big(\frac{w^{(t)}_{1i}+w^{(t)}_{2i}}{||\bar{w}^{(t)}||}\Big) - \phi\Big(\frac{w^{(t)}_{1i}-w^{(t)}_{2i}}{||\bar{w}^{(t)}||}\Big)\Big| \leq 2|w^{(t)}_{2i}| \cdot \max_{|\delta| \leq |w^{(t)}_{2i}|} \frac{1}{||\bar{w}^t||} \varphi \Big(\frac{w^{(t)}_{1i}+\delta}{||\bar{w}^{(t)}_i||}\Big) \overset{(a)}{\leq} \frac{2c_0 (1+\hat{c})^t}{\sqrt{dk}\log d}  \\
&\Big| \phi\Big(\frac{w^{(t)}_{1i}+w^{(t)}_{2i}}{||\bar{w}^{(t)}||}\Big) - \phi\Big(\frac{w^{(t)}_{1i}}{||\bar{w}^{(t)}||}\Big)\Big| \leq |w^{(t)}_{2i}| \cdot \max_{|\delta| \leq |w^{(t)}_{2i}|} \frac{1}{||\bar{w}^t||} \varphi \Big(\frac{w^{(t)}_{1i}+\delta}{||\bar{w}^{(t)}_i||}\Big) \overset{(a)}{\leq} \frac{c_0 (1+\hat{c})^t}{\sqrt{dk}\log d}  \\
& \frac{1}{||\bar{w}^t||} \varphi \Big(\frac{w^{(t)}_{1i}}{||\bar{w}^{(t)}_i||}\Big) \leq \max_{|\delta| \leq |w^{(t)}_{2i}|} \frac{1}{||\bar{w}^t||} \varphi \Big(\frac{w^{(t)}_{1i}+\delta}{||\bar{w}^{(t)}_i||}\Big) \overset{(a)}{\leq} \frac{c_0 (1+\hat{c})^t}{\sqrt{dk}\log d}
\end{align} \normalsize
where $(a)$ is due to equation \eqref{eq:noise_ih}.
\end{proof}

\subsection{Closed-form Gradient Expressions}
\label{sec:proof_gradients}

In this section, we provide closed-form expressions for gradients along the linear, slab and noise coordinates: $\grad_{w_{1i}} \loss_f(\S_n)$, $\grad_{w_{2i}} \loss_f(\S_n)$ and $\grad_{\bar{w}_i} \loss_f(\S_n)$. First, we provide a closed-form expression for the gradient along the linear coordinate:
\begin{lemma}
\label{lemma:grad_linear}
If $n > cd^2$ and $y_i f(x_i) < 1 \, \forall (x_i, y_i) \in \S_n$, then w.p. greater than $1-\frac{3}{n}$:
\begin{align*}
  \grad_{w_{1i}} \loss_f(\S_n) = \minus \frac{v_i}{4}\Big[2+\phi\Big(\frac{w_{1i}+w_{2i}}{||\bar{w}_i||}\Big) + \phi\Big(\frac{w_{1i} \minus w_{2i}}{||\bar{w}_i||}\Big) - 2\phi\Big(\frac{ w_{1i}}{||\bar{w}_i||}\Big) \Big] \pm \frac{5v_i}{d} \sqrt{\frac{\log (cd^2)}{c}}
\end{align*}
\end{lemma}
\begin{proof}
  \small
  \begin{align*}
  \grad_{w_{1i}} \loss_f(\S_n) &= \minus \frac{v_i}{n} \sum^n_{j=1} \ind{y_jf(x_j) \leq 1}\ind{w^T_i x_j \geq 0}y_j x_{1j} \\
  &\overset{(a)}= \minus \frac{v_i}{n} \sum^n_{j=1} \ind{\bar{w}^T_i\bar{x}_j \geq -w_{i1}y_j - w_{i2}\ind{y_j=1}\varepsilon_j} \\
  &\overset{(b)}{=} \minus \frac{v_i}{n} \sum^n_{j=1} \ind{Z_j \geq \frac{-w_{1i}y-w_{2i}\ind{y_j=1}\varepsilon_j}{||\bar{w}_i||}} & \text{where } Z_j \sim \mathcal{N}(0,1) \\
  &= \minus v_i \sum^{\{0,1,-1\}}_{l} \frac{1}{n} \sum^n_{i=1} \ind{x_{2j}=l}\ind{Z_j \geq \frac{\minus w_{1i}(2l^2\minus 1)\minus w_{2i}l}{||\bar{w}_i||}} \\
  &{=} \minus v_i \sum^{\{0,1,\minus 1\}}_{l}  \Big(\pr(x_{2j}=l)\phi \Big(\frac{w_{1i}(2l^2\minus 1)+w_{2i}l}{||\bar{w}_i||}\Big) \pm \sqrt{\frac{\log n}{n}}\Big) & \text{via lemma } \ref{lemma:sum_cdf} \\
  & {=} \minus \frac{v_i}{4} [ \phi(\frac{w_{1i}+w_{2i}}{||\bar{w}_i||}) + \phi(\frac{w_{1i} \minus w_{2i}}{||\bar{w}_i||}) + 2\phi(\frac{\minus w_{1i}}{||\bar{w}_i||})] \pm \frac{5v_i}{d} \sqrt{\frac{\log (cd^2)}{c}} & n > cd^2 \\
  &= \minus \frac{v_i}{4}\Big[2+\phi\Big(\frac{w_{1i}+w_{2i}}{||\bar{w}_i||}\Big) + \phi\Big(\frac{w_{1i} \minus w_{2i}}{||\bar{w}_i||}\Big) - 2\phi\Big(\frac{ w_{1i}}{||\bar{w}_i||}\Big) \Big] \pm \frac{5v_i}{d} \sqrt{\frac{\log (cd^2)}{c}}  &\text{ w.p. } 1-\frac{3}{n}
  \end{align*}
  \normalsize
where $(a)$ is due to $y_i x_{i1} = y^2_i = 1$ \& $\ind{y_j f(x_j) \leq 1}=1$ and $(b)$ is due to $\ind{\bar{w}^T_i \bar{x}_j \geq k} = \ind{||\bar{w}_i|| Z_j \geq k}$.
\end{proof}

Similarly, we provide a closed-form expression for the gradient along the slab coordinate:
\begin{lemma}
\label{lemma:grad_slab}
If $n > cd^2$ and $y_i f(x_i) < 1 \, \forall (x_i, y_i) \in \S_n$, then w.p. greater than $1-\frac{3}{n}$:
\begin{align*}
  &\grad_{w_{2i}} \loss_f(\S_n) = \minus \frac{v_i}{4} \Big[\phi\Big(\frac{w_{1i}+w_{2i}}{||\bar{w}_i||}\Big) - \phi\Big(\frac{w_{1i}-w_{2i}}{||\bar{w}_i||}\Big) \Big] \pm \frac{5v_i}{d} \sqrt{\frac{\log (cd^2)}{c}}
\end{align*}
\end{lemma}
\begin{proof}
  \begin{align*}
  \grad_{w_{2i}} \loss_f(\S_n) &= \minus \frac{v_i}{n} \sum^n_{j=1} \ind{y_jf(x_j) \leq 1}\ind{w^T_i x_j \geq 0}y_j x_{2j} \\
  &\overset{(a)}= \minus \frac{v_i}{n} \sum^n_{j=1} \ind{\bar{w}^T_i\bar{x}_j \geq -w_{1i}y_j - w_{2i}\ind{y_j=1}\varepsilon_j}  \ind{y_j=1} \varepsilon_j \\
  &\overset{(b)}= v_i \sum^{\{\minus 1,1\}}_l \frac{1}{n} \sum^n_{j=1} (-1)^{\ind{\varepsilon_j=l}} \ind{Z_j \geq \frac{-w_{1i}-w_{2i}l}{||\bar{w}_i||}} \\
  &= \minus v_i \sum^{\{1,\minus 1\}}_{l}  \Big(\pr(x_{2j}=l)\phi \Big(\frac{w_{1i}+w_{2i}l}{||\bar{w}_i||}\Big) \pm \sqrt{\frac{\log n}{n}}\Big) & \text{via lemma} ~\ref{lemma:sum_cdf} \\
  &= \minus \frac{v_i}{4} [ \phi(\frac{w_{1i}+w_{2i}}{||\bar{w}_i||}) - \phi(\frac{w_{1i} \minus w_{2i}}{||\bar{w}_i||})] \pm \frac{5v_i}{d} \sqrt{\frac{\log (cd^2)}{c}} & n > cd^2 \\
  &= \minus \frac{v_i}{4} \Big[\phi\Big(\frac{w_{1i}+w_{2i}}{||\bar{w}_i||}\Big) - \phi\Big(\frac{w_{1i}-w_{2i}}{||\bar{w}_i||}\Big) \Big] \pm \frac{5v_i}{d} \sqrt{\frac{\log (cd^2)}{c}} & \text{ w.p. } 1-\frac{3}{n}
  \end{align*}
  \normalsize
where $(a)$ is due to $y_i x_{i2} = \ind{y_i=1}\varepsilon_i$ \& $\ind{y_j f(x_j) \leq 1}=1$ and $(b)$ is due to $\ind{\bar{w}^T_i \bar{x}_j \geq k} = \ind{||\bar{w}_i|| Z_j \geq k}$.
\end{proof}

Next, we provide a closed-form expression for the gradient along the noise coordinates:
\begin{lemma}
\label{lemma:grad_noise}
If $n > cd^2$ and $y_i f(x_i) < 1 \, \forall (x_i, y_i) \in \S_n$, then w.p. greater than $1-\frac{1}{3n}$:
\small
\begin{align*}
  &\grad_{\bar{w}_{i}} \loss_f(\S_n) = \bar{\G}\bar{w}_i \pm \frac{3|v_i|\log (\sqrt{c}d)}{\sqrt{c}d} \frac{\bar{w}_i}{||\bar{w}_i||} \pm \frac{6|v_i|}{\sqrt{cd}} u^{\perp}_i \\
  &\bar{\G} = \minus \frac{v_i}{4 ||\bar{w}_i||}\Big[\varphi\Big(\frac{w_{1i}+w_{2i}}{||\bar{w}_i||}\Big) + \varphi\Big(\frac{w_{1i}\minus w_{2i}}{||\bar{w}_i||}\Big) - 2\varphi\Big(\frac{w_{1i}}{||\bar{w}_i||}\Big) \Big]
\end{align*}
\normalsize
where $u^{\perp}_i$ is some unit vector orthogonal to $\bar{w}_i$.
\end{lemma}
\begin{proof}
Let $S \subset \R^{d-2}$ denote the subspace spanned by $\bar{w}_i$. Then, for any $x \in \R^{d}$, $x = x^S + x^{S^{\perp}}$ where $x^S$ \& $x^{S^{\perp}}$ are the orthogonal projections of $x$ onto $S$ and its orthogonal complement $S^{\perp}$. We show the $\ell_2$ norm of the orthogonal projections of $\grad_{\bar{w}_{i}} \loss_f(\S_n)$ onto $S$ and $S^{\perp}$ are $O(\frac{1}{\sqrt{d}})$:
\small
\begin{align*}
\grad_{\bar{w}_{i}} \loss_f(\S_n) &= \minus \frac{v_i}{n} \sum^n_{j=1} \ind{y_jf(x_j) \leq 1}\ind{w^T_i x_j \geq 0}y_j \bar{x}_{j} \\
&= \underbrace{\minus \frac{v_i}{n} \sum^n_{j=1} \ind{w^T_i x_j \geq 0}y_j (\bar{x}^S_{j})}_{\text{case 1}} \underbrace{- \frac{v_i}{n} \sum^n_{j=1} \ind{w^T_i x_j \geq 0}y_j (\bar{x}^{S^{\perp}}_{j})}_{\text{case 2}}
\end{align*}
\normalsize
Next, we show that the projection of $\grad_{\bar{w}_{i}} \loss_f(\S_n)$ onto $S^{\perp}$ (i.e., case 2) has small norm w.p. greater than $1-\frac{1}{d}$:
\small
\begin{align*}
  ||\grad_{\bar{w}_{i}} \loss_f(\S_n))^{S^{\perp}}|| &= ||\frac{v_i}{n} \sum^n_{j=1} \ind{w^T_i x_j \geq 0}y_j \bar{x}^{S^{\perp}}_{j}|| = ||\frac{v_i}{n} \sum^n_{j=1} \ind{w^T_i x^S_j \geq 0}y_j \bar{x}^{S^{\perp}}_{j}|| \\
  & \overset{(a)}\leq |v_i| \cdot || \sum^n_{j=1} \mathcal{N}(0, \frac{1}{n^2} I_{d-2}) || = |v_i| \cdot || \mathcal{N}(0, \frac{1}{n} I_{d-2}) || \\
  & \overset{(b)}\leq 4|v_i|\sqrt{\frac{d}{n}} \pm 2 |v_i|\sqrt{\frac{\log n}{n}} \overset{(c)}{\leq} \frac{6|v_i|}{\sqrt{cd}} & \text{ w.p. } 1-\frac{1}{n}
\end{align*} \normalsize
where $(a)$ is because $x^S_j \perp \bar{x}^{S^{\perp}_j}$, $(b)$ is via fact \ref{lemma:norm_l2} and $(c)$ is due to $n \geq cd^2$. Next, we show that the norm of the gradient in the direction of $\bar{w}_i$ (i.e., case 1) is close to $\bar{\G}$ w.h.p.:
\small
\begin{align*}
  &\grad_{\bar{w}_{i}} \loss_f(\S_n))^{S} = \minus \frac{v_i}{n} \sum^n_{j=1} \ind{w^T_i x_j \geq 0}y_j \bar{x}^S_{j} \\
  &\overset{(a)}= \minus \Big(\frac{1}{n} \sum^n_{j=1} \ind{w^T_ix^S_j \geq 0} y_j \bar{w}^T_i \bar{x}_j \Big)\frac{v_i \bar{w}_i}{||\bar{w}_i||^2} \\
  &\overset{(b)}= \minus \Big(\frac{1}{n} \sum^n_{j=1} \ind{Z_j \geq \frac{-w_{1i}y-w_{2i}\ind{y_j=1}\varepsilon_j}{||\bar{w}_i||}} y_j Z_j \Big)\frac{v_i \bar{w}_i}{||\bar{w}_i||} \\
  &= (\sum^{\{0,\pm 1\}}_{l} \frac{(-1)^{\ind{l\neq 0}}}{n} \sum^n_{i=1} \ind{x_{2j}=l \land Z_j \geq \frac{\minus w_{1i}(2l^2\minus 1)\minus w_{2i}l}{||\bar{w}_i||}} Z_j )\frac{v_i \bar{w}_i}{||\bar{w}_i||} \\
  &= \Big[2\varphi(\frac{w_{i1} }{||\bar{w}_i||})-\varphi(\frac{w_{1i} + w_{2i} }{||\bar{w}_i||})-\varphi(\frac{w_{1i} - w_{2i} }{||\bar{w}_i||}) \pm \frac{5\log n}{\sqrt{n}} \Big] \frac{v_i \bar{w}_i}{4||\bar{w}_i||} & \text{via lemma } \ref{lemma:norm_bern_sum} \\
  &=  \bar{\G}\bar{w}_i \pm \frac{3|v_i|\log (\sqrt{c}d)}{\sqrt{c}d} \frac{\bar{w}_i}{||\bar{w}_i||} & \text{ w.p. } 1 - \frac{12}{n}
\end{align*}
\normalsize
where $(a)$ is because $\bar{x}^S_j = \frac{\bar{w}^T_ix_j}{||\bar{w}_i||^2}\bar{w}_i$ and $(b)$ is because $(b)$ is due to $\ind{\bar{w}^T_i \bar{x}_j \geq k} = \ind{||\bar{w}_i|| Z_j \geq k}$. Therefore, by combining the results in case 1 and 2, the following holds w.p. greater than $1-\frac{13}{n}$:
\begin{align*}
\grad_{\bar{w}_{i}} \L_f(\S_n) = \bar{\G}\bar{w}_i \pm \frac{3|v_i|\log (\sqrt{c}d)}{\sqrt{c}d} \frac{\bar{w}_i}{||\bar{w}_i||} \pm \frac{6|v_i|}{\sqrt{cd}} u^{\perp}_i
\end{align*}
\end{proof}

\subsection{Miscellaneous Lemmas}
\label{sec:proof_misc}

\begin{lemma}
\label{lemma:norm_max}
Let $X_i \sim \mathcal{N}(0, \sigma^2)$ and $\delta \in (0,1)$. Then, $\max_{i \in [k]} |X_i| \leq \sigma \sqrt{2\log(\frac{2k}{\delta})}$ with probability greater than $1-\delta$,
\end{lemma}
\begin{proof}
  Let $\varphi$ denote the probability density function of the standard normal.
	Also let $Z \sim \mathcal{N}(0, 1)$. Then, for $t \geq 1$, we have:
  \begin{align*}
    \PR{|X| \geq \sigma t} &= \PR{|Z| \geq t} = 2 \int^{\infty}_{t} x \varphi(x) \, dx
    \leq \frac{2}{t} \int^{\infty}_{t} x \varphi(x) \, dx \overset{(a)} \leq \frac{2}{t} \int^{t}_{\infty} \varphi'(x) \, dx
    \leq \frac{2}{t} \varphi(t) \leq 2 \varphi(t)
  \end{align*}
  where $(a)$ is because $\varphi'(x)=\minus x \varphi(x)$. Using union bound with $t = \sqrt{2 \log(\frac{2k}{\delta})} \geq 1 \, \forall \delta \in (0,1)$ gives the desired result.
\end{proof}

\begin{lemma}
\label{lemma:norm_pdf}
  Let $\phi$ and $\varphi$ denote the cumulative distribution function and the
  probability density function of the standard gaussian. Then, for any $Z \sim \mathcal{N}(0,1)$ and $k \in \mathds{R}$:
  \begin{align*}
    \mathds{E}[\ind{Z \geq k}Z] = \varphi(k) = \exp(\sfrac{\minus k^2}{2})
  \end{align*}
\end{lemma}
\begin{proof}
  The expectation $\mathds{E}[\ind{Z \geq k}Z]$ can be simplified as follows:
  \begin{align*}
    \mathds{E}[\ind{Z \geq c}Z] &= \Pr[Z \geq c] \E{Z|Z \geq c}
    = \phi^c(k) \int^{\infty}_{k} x \frac{\varphi(x)}{\phi^c(k)} \, dx
    \overset{(a)}{=} \minus \int^{\infty}_k \varphi'(x) \, dx = \varphi(k)
  \end{align*}
  where $(a)$ is due to $\varphi'(x)=\minus x \varphi(x)$.
\end{proof}

\begin{lemma}
\label{lemma:sum_cdf}
  Let $b_i \sim \text{bernoulli}(p)$ and $Z_i \sim \mathcal{N}(0,1)$. Let $X_i = b_i \ind{Z_i \geq k}$ and $\bar{X} = \frac{1}{n} \sum^n_{i=1} X_i$. Then:
	$$\Pr\Big(|\bar{X}-p\phi(-k)| \geq \sqrt{\frac{\log n}{n}} \Big) \leq \frac{1}{n}$$
\end{lemma}
\begin{proof}
  Note that $\E{\bar{X}}=\E{X_i}=\E{b_i}\E{\ind{Z_i \geq k}} = p\phi(-k)$ and $|X_i| \leq 1$. Therefore, using Hoeffding's inequality with $t=\sqrt{\frac{\log n}{n}}$ directly gives the result.
\end{proof}

\begin{lemma}
\label{lemma:norm_bern_sum}
  Let $b_i \sim \text{bern}(p)$ and $Z_i \sim \mathcal{N}(0,1)$. Let $X_i = b_i \ind{Z_i \geq k}Z_i$ and $\bar{X} = \frac{1}{n} \sum^n_{i=1} X_i$. Then:
  \begin{align*}
    \pr\Big(|\bar{X}-p\varphi(k)| \leq \sqrt{\frac{2}{n}} \log n \Big) \geq 1-\frac{4}{n}
  \end{align*}
\end{lemma}
\begin{proof}
  Since $|X_i| = |b_i \ind{Z_i \geq k}Z_i| \leq |Z_i|$, we have $\max_{i \in [n]} |X_i| \leq \sqrt{4 \log (n)}$ w.p. at least $1-\frac{2}{n}$ via lemma \ref{lemma:norm_max}. From lemma \ref{lemma:norm_pdf}, we get $\E{X_i} = \E{b_i}\E{\ind{Z_i \geq k}Z_i} = p\varphi(k)$.
  Let $A = \ind{|X_i| \leq \sqrt{4 \log (n)} \, \forall i \in [n]}$. Given $A$, we can use Hoeffding's inequality with $t^* = \sqrt{\frac{2}{n}} \log n$ (and $\delta=\sfrac{2}{n}$) to get the desired result, as follows:
  \begin{align*}
    \pr ( |\bar{X}-p\varphi(k)| \leq t^*) &\geq
    \pr ( |\bar{X}-p\varphi(k)| \leq t^* \, | A) \pr(A)
    \geq (1-\frac{2}{n})^2 \geq 1-\frac{4}{n}
  \end{align*}
Therefore, $\bar{X} = p\varphi(k) \pm \sqrt{\frac{2}{n}}\log n$ w.p. at least $1-\frac{4}{n}$.
\end{proof}

\begin{lemma}
\label{lemma:maxima}
Let $g: \mathds{R}\textbackslash \{0\}\to \mathds{R}$ be defined as $g_z(x) = \frac{1}{x}\exp(\minus \frac{z^2}{2x^2})$. Then, (1) $|z|$ and $\minus |z|$ are the global maximizer and minimizer respectively, and (2) $g$ monotonically increases from $\minus |z|$ to $|z|$.
\end{lemma}
\begin{proof}
Note that $g'_z(x) = \frac{1}{x^2}\exp(\minus \frac{z^2}{2x^2})(\frac{z^2}{x^2}-1)$. Therefore, the critical points of $g$ are $|z|$ and $\minus |z|$. Let $S=\{t: |t| \geq |z|, t \in \mathds{R}/\{0\}\}$. Note that $g'_z(x) < 0$ for all $x \in S$ and $g'_z(x) > 0$ for all $x \in S^c$. Therefore, (1) and (2) hold.
\end{proof}

\begin{fact}
\label{lemma:norm_l2}
  Let $X \sim \mathcal{N}(0, \sigma^2 I_d)$ denote a $d$-dimensional gaussian
  vector. Then, from \cite{wainwright2019high}, w.p. greater than $1-\delta$:
  $$||X||_2 \leq 4 \sigma \sqrt{d} + 2\sigma$$
\end{fact}

\end{document}